\documentclass{article}

\PassOptionsToPackage{numbers, compress}{natbib}

    \usepackage[final]{neurips_2021}

\usepackage[numbers]{natbib}

\usepackage[utf8]{inputenc} 
\usepackage[T1]{fontenc}    
\usepackage{hyperref}       
\hypersetup{
    colorlinks=true,
    linkcolor=blue,
    filecolor=magenta,      
    urlcolor=blue,
    citecolor=blue,
}
\usepackage{url}            
\usepackage{booktabs}       
\usepackage{amsfonts}       
\usepackage{nicefrac}       
\usepackage{microtype}      
\usepackage{xcolor}         
\usepackage{amsmath}
\usepackage{mathtools}
\usepackage{amssymb}
\usepackage{amsthm}
\usepackage{pifont}
\usepackage[linesnumbered,ruled]{algorithm2e} 
\usepackage{algorithmic}  
\usepackage{enumitem}
\SetKwRepeat{Do}{do}{while}
\usepackage{bm}
\usepackage{subcaption}
\usepackage{cases}
\usepackage{graphicx}
\usepackage{wrapfig}
\usepackage{empheq} 
\usepackage[most]{tcolorbox} 
\usepackage{siunitx}

\usepackage{selectp}

\newcommand{\bigO}{\mathcal{O}}

\newcommand{\R}{{\mathbb{R}}}

\DeclareMathOperator*{\argmin}{arg\,min}

\newcommand{\T}{\mathcal{T}}

\newcommand{\W}{\mathcal{W}}
\newcommand{\matL}{\mathcal{L}}
\newcommand{\hmatL}{\hat{\mathcal{L}}}
\newcommand{\hF}{\hat{F}}
\newcommand{\tG}{\tilde{G}}
\newcommand{\matS}{\mathcal{S}}
\newcommand{\matH}{\mathcal{H}}

\newcommand{\A}{\mathcal{A}}
\newcommand{\F}{\mathcal{F}}

\newcommand{\B}{\mathcal{B}}
\newcommand{\D}{\mathcal{D}}
\newcommand{\E}{\mathbb{E}}

\newcommand{\Z}{\mathcal{Z}}

\newcommand{\tein}{\text{in}}
\newcommand{\teout}{\text{out}}

\newtheorem{theorem}{Theorem}  
\newtheorem{definition}{Definition}

\newtheorem{proposition}{Proposition}
\newtheorem{lemma}{Lemma}
\newtheorem{remark}{Remark}
\newtheorem{corollary}{Corollary}

\newtheorem{assumption}{Assumption}

\newcommand{\beq}{\begin{equation}}
\newcommand{\eeq}{\end{equation}}
\newcommand{\beqa}{\begin{eqnarray}}
\newcommand{\eeqa}{\end{eqnarray}}
\newcommand{\beqs}{\begin{equation*}}
\newcommand{\eeqs}{\end{equation*}}
\newcommand{\beqas}{\begin{eqnarray*}}
\newcommand{\eeqas}{\end{eqnarray*}}

\DeclarePairedDelimiter\floor{\lfloor}{\rfloor}

\title{Generalization of Model-Agnostic Meta-Learning Algorithms: Recurring and Unseen Tasks}

\author{%
  Alireza Fallah \\
   EECS Department \\
   Massachusetts Institute of Technology \\
  \texttt{afallah@mit.edu} \\
   \And
   Aryan Mokhtari \\
   ECE Department \\
   The University of Texas at Austin \\
   \texttt{mokhtari@austin.utexas.edu} \\
   \And
   Asuman Ozdaglar \\
   EECS Department \\
   Massachusetts Institute of Technology \\
   \texttt{asuman@mit.edu} \\
}

\begin{document}

\maketitle

\begin{abstract}
In this paper, we study the generalization properties of Model-Agnostic Meta-Learning (MAML) algorithms for supervised learning problems. We focus on the setting in which we train the MAML model over $m$ tasks, each with $n$ data points, and characterize its generalization error from two points of view: First, we assume the new task at test time is one of the training tasks, and we show that, for strongly convex objective functions, the expected excess population loss is bounded by ${\bigO}(1/mn)$. Second, we consider the MAML algorithm's generalization to an unseen task and show that the resulting generalization error depends on the total variation distance between the underlying distributions of the new task and the tasks observed during the training process. Our proof techniques rely on the connections between algorithmic stability and generalization bounds of algorithms. In particular, we propose a new definition of stability for meta-learning algorithms, which allows us to capture the role of both the number of tasks $m$ and number of samples per task $n$ on the generalization error of MAML. 
\end{abstract}

\section{Introduction}
In several machine learning problems, it is of interest to design algorithms that can be adjusted based on previous experiences and tasks to perform better on a new task. In particular, meta-learning algorithms achieve such a goal through various approaches, including finding a proper meta-initialization for the new task \cite{finn17a, Reptile, khodak2019adaptive}, updating the model architecture \cite{baker2016designing, zoph2016neural, zoph2018learning}, or learning the parameters of optimization algorithms \cite{ravi2016optimization, NIPS2016_Andry}.

A popular meta-learning framework that has shown promise in practice is Model-Agnostic Meta-Learning (MAML), which was first introduced in \cite{finn17a}. MAML algorithm uses available training data on a number of tasks to come up with a meta-initialization that performs well after it is slightly updated at test time with respect to the new task. 
In other words, unlike standard supervised learning, in which we aim to find a model that generalize well to a new task \textit{without any adaptation step}, in MAML our goal is to find an initial model for learning a new task when \textit{we have access to limited labeled data for that task} to run one (or a few) step(s) of stochastic gradient descent (SGD). 

As shown in Fig.~\ref{fig_MAML}, in MAML we are given $m$ tasks with $m$ corresponding datasets $\{\matS_i\}_{i=1}^m$ in the training phase. Once the model is trained ($w_{\text{train}}^*$), a new task is revealed at test time for which we have access to $K$ \textit{labeled} samples drawn from $\D_{\text{test}}$. We use these labeled samples of the new task to update the trained model by running a step of SGD leading to a new model for the test task ($w_{\text{new}}^*$). We finally evaluate the performance of the updated model over the test task, denoted by $\mathcal{L}_{test}(w_{\text{new}}^*)$. 

MAML and its variants have been  extensively studied over the past few years from both empirical and theoretical point of view \cite{Reptile, MAML++, Meta-SGD, grant2018recasting, alpha-MAML, fallah2019convergence, xu2020meta, ji2020multi, wang2020global}. 
In particular,  \cite{fallah2019convergence} provided convergence guarantees
\begin{wrapfigure}{r}{0.4\textwidth}
\centering
\input{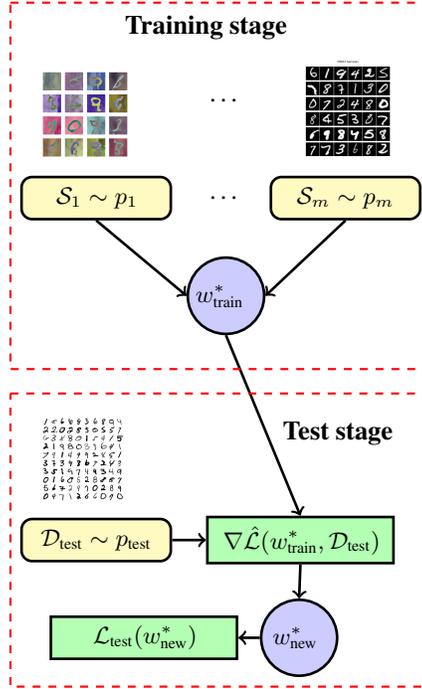}
    \caption{\label{fig_MAML} MAML framework}
    \vspace{-2mm}
\end{wrapfigure}
 for MAML algorithm under the assumption that access to
 fresh samples at any round of the training stage is possible, and \cite{ji2020multi} extended this results to the case that multiple gradient steps can be performed at test time. 
However, one shortcoming of such analysis is that, at training stage, we often do not have access to fresh samples at every iteration. Instead, we have access to a large set of realized samples and we typically do multiple passes over the data points during the training stage. 

Hence, it is essential to come up with a novel analysis that addresses this issue by characterizing the training error and generalization error of MAML separately. In this paper, we accomplish this goal and showcase the role of different problem parameters in the generalization error of MAML. Specifically, we assume that we are given $m$ supervised learning tasks, with (possibly different) underlying distributions $p_1, \ldots p_m$, where for each task we have access to $n$ samples\footnote{More precisely, in our analysis we take $2n$ samples per each task to simplify derivations.}. As we measure the performance of a model by its loss after one step of SGD adaptation with $K$ samples, the problem that one can solve in the training phase is minimizing the average loss, over all given $m$ tasks and their $n$ samples, after one step of SGD with $K$ samples. This empirical loss can be considered as a surrogate for the desired expected loss (with respect to tasks data) over all $m$ tasks. Here, we focus on the case that MAML is used to solve this empirical minimization problem, and our goal is to quantify the test error of MAML output. To tackle this problem, we first briefly revisit the results from the optimization literature to bound the training error of MAML, assuming that the loss functions are strongly convex. We next turn to the main focus of our paper which is the generalization properties of MAML. More specifically, we address the following questions:

$\bullet$ \textit{If one of the $m$ given tasks recurs uniformly at random at test time, then how well (in expectation) would the trained model perform after adaptation with SGD over the fresh samples of that task?} In other words, having training error minimized, what would be the generalization error and our guarantee on test error? Here, we show that for strongly convex objective functions, we could achieve a generalization error that decays at $\bigO(1/mn)$. Our analysis builds on the connections between algorithmic stability and generalization of the output of algorithms. While this relation is well-understood in classic statistical learning \cite{bousquet2002stability, hardt2016train}, here we propose a novel stability definition for meta-learning algorithms which allows us to restore such connection for our setting.

$\bullet$ \textit{Assuming that the task at test time is NOT one of the $m$ tasks at training, how would the model perform on that task after the adaptation step?} We answer this question by focusing on the case that the revealed task at the  test time is a new \textit{unseen} task with underlying data distribution $p_{m+1}$, and formally characterizing the generalization error of MAML in this case. We show that when the task at test time is new, the generalization error also depends on the total variation distance between $p_{m+1}$ and $p_1,\dots,p_m$.
 
\textbf{Related work:}
Recently, there has been significant progress in studying theoretical aspects of meta-learning, in particular, MAML. Authors in \cite{NEURIPS2019_072b030b} proposed iMAML which updates the model using an approximation of one step of proximal point method and studied its convergence. In  
\cite{collins2020task}, authors introduced the task-robust MAML by considering a minimax formulation rather than minimization. Several papers have also studied MAML through more general frameworks such as bilevel optimization \cite{likhosherstov2020ufo}, stochastic compositional optimization \cite{chen2020solving}, and conditional stochastic optimization \cite{Hu2020BiasedSG}. Also, several works have studied the extension of meta-learning theory to online learning \cite{finn19a, khodak2019adaptive}, federated learning \cite{fallah2020personalized}, and reinforcement learning \cite{liu2019taming, fallah2020provably}.

The most relevant paper to our work is \cite{chen2020closer} that studies generalization of meta-learning algorithms using stability techniques and shows a $\bigO(1/\sqrt{m})$ bound for nonconvex loss functions. Here we focus on strongly convex objective functions and present an analysis that differs from this work in two fundamental aspects.
First, we present a different notion of stability that allows us to capture the number of data points per task in our bound. In particular, our stability notion measures sensitivity of the algorithm to perturbations that involve changing $K$ data points which is the data unit involved in the adaptation step of the MAML algorithm. This enables us to obtain a much tighter bound $\bigO(1/mn)$ (compared to $\bigO(1/m)$ achieved in \cite{chen2020closer} for strongly convex functions), highlighting the dependence on the number of the data samples available for each task. Second, we also consider the generalization of MAML for the case that the task at test time is not one of the available tasks during the training stage.

The generalization of MAML has also been studied in \cite{guiroy2019towards} from an empirical point of view. In particular, they show that the generalization of MAML to new tasks is correlated with the coherence between their adaptation trajectories in parameter space. This is aligned with the connection of generalization and closeness of underlying distributions that we observe in our results.

\section{Problem formulation}\label{sec:Formulation}
In this paper, we consider the supervised learning setting, where each data point is denoted by $z=(x,y) \in \Z$ with $x \in \mathcal{X}$ being the input (feature vector) and $y \in \mathcal{Y}$ being its corresponding  label. We use the loss function
$l: \R^d \times \Z \to \R^{+}$ 
to evaluate the performance of a model parameterized by $w \in \W$, where $\W$ is a convex and closed subset of $\R^d$. In other words, for a data point $z = (x,y) \in \Z$, the loss $\ell(w,z)$ denotes the error of model $w$ in predicting the label $y$ given input $x$.

We consider access to $m$ tasks denoted by $\T_1, \ldots, \T_m$, where the data corresponding to each task $\T_i$ is generated from a distinct distribution $p_i$. The \textit{population loss} corresponding to task $\T_i$ for model $w$ is defined as
$\matL_i(w):= \E_{z \sim p_i}[\ell(w,z)].$

We further use the notation $\hmatL(w;\D)$ to denote the \textit{empirical loss} corresponding to dataset $\D$, which is defined as the average loss of $w$ over the samples of dataset $\D$, i.e., 
$\hmatL(w;\D) := \frac{1}{|\D|} \sum_{z \in \D} \ell(w,z),
$ 
where $|\D|$ is the size of dataset $\D$. In general, and throughout the paper, we use the \textit{hat notation} to distinguish \textit{empirical} losses from population losses. 

Our goal is to find $w \in \W$ that performs well on average\footnote{Our analysis can be extended to the case that the distribution over tasks is not uniform.} over all tasks, after it is updated with respect to the new task and by using one step of stochastic gradient descent (SGD) with a batch of size $K$. To formally introduce this problem we first define the function $F_i(w)$ which captures the performance of model $w$ over task $\T_i$ once it is updated by a single step of SGD,
\begin{align}\label{eqn:population_prob:b}
F_i(w):= \E_{\D_i^{\text{test}}} \left[ \matL_i \left (w - \alpha \nabla \hmatL(w, \D_i^{\text{test}}) \right ) \right] 
= \E_{\D_i^{\text{test}}} \E_{z \sim p_i} \!\bigg [ \ell \bigg (w - \frac{\alpha}{K} \!\sum_{z' \in \D_i^{test}}\! \nabla \ell(w,z'),z \bigg ) \bigg]    
\end{align}
where $\D_i^{\text{test}}$ is a batch with $K$ different samples, drawn from the probability distribution $p_i$. Note that the outer expectation is taken with respect to the choice of elements of $\D_i^{\text{test}}$ while the inner one is taken with respect to the data of task $i$.

As our goal is to find a model that performs well after one step of adaptation over all $m$ tasks, we minimize the average expected loss over all given tasks, which can be written as 
\begin{equation}\label{eqn:population_prob}
    \min_{w \in \W} F(w) := \frac{1}{m} \sum_{i=1}^m F_i(w). 
\end{equation}
As the underlying distribution of tasks are often unknown in most applications, we are often unable to directly solve the problem in \eqref{eqn:population_prob}.
 On the other hand, for each task, we often have access to data points that are drawn according to their data distributions. Therefore, instead of solving \eqref{eqn:population_prob}, we solve its sample average surrogate problem in which each $F_i$ is approximated by its empirical loss.

To formally define the empirical loss for each task, suppose for each task $\T_i$ we have access to a training set $\matS_i$, where its elements are drawn independently according to the probability distribution  $p_i$. We further divide the set $\matS_i$ into two disjoint sets of size $n$ defined as $\matS_i^\tein$ and $\matS_i^\teout$, i.e., $\matS_i := \{\matS_i^\tein, \matS_i^\teout\}$ and $|\matS_i^\tein|=|\matS_i^\teout|=n$.
Here, we use the elements of the $\matS_i^{\tein}$ to estimate the inner gradient $\nabla \hmatL(w, \D_i^{\text{test}})$ and use the samples in the set $\matS_i^{\teout}$ to estimate the outer function $\matL_i(.)$. Specifically, we define the sample average of $F_i$ using data sets $\matS_i^\tein$ and $\matS_i^\teout$ as
\begin{align}\label{eqn:empirical_prob:b}
\hF_i(w, \matS_i) \!:&=\! \frac{1}{\binom{n}{K}}\! \sum_{\substack{\D_i^\tein \subset \matS_i^\tein |\D_i^\tein| = K}} \!\hmatL \left (w - \alpha \nabla \hmatL(w,\D_i^\tein), \matS_i^\teout \right) \\
 & = \frac{1}{\binom{n}{K}} \sum_{\substack{\D_i^\tein \subset \matS_i^\tein |\D_i^\tein| = K}}
 \frac{1}{n} \sum_{\substack{z\in \matS_i^{\teout}}}\! \ell \left (w -   \frac{\alpha}{K} \!\sum_{z' \in \D_i^\tein} \nabla \ell(w,z'),z \right ). \nonumber
\end{align}
This expression shows that we use all $n$ elements of $\matS_i^\teout$ to approximate the expectation  required for the computation of $\matL_i$, and we approximate the expectation with respect to the test set by averaging over all subsets of $\matS_i^\tein$ that have $K$ elements. Given this expression, the sample average approximation (empirical loss) of Problem~\eqref{eqn:population_prob:b} is given by
\begin{equation}\label{eqn:empirical_prob}
\argmin_{w \in \W} \hF(w, \matS):= \frac{1}{m} \sum_{i=1}^m \hF_i(w, \matS_i),
\end{equation}
where $\matS := \{\matS_i\}_{i=1}^m$ is defined as the concatenation of all tasks data sets.

Having the dataset $\matS$, a (possibly randomized) optimization algorithm $\A$ with output $\A(\matS)$ can be used to find an approximate solution to the problem in \eqref{eqn:empirical_prob}. The error of this solution with respect to the MAML empirical loss, i.e., $ \hF(\A(\matS), \matS) - \min_\W \hF(., \matS)$, is called \textit{training error}. In this paper, we are mainly interested to bound the \textit{test error} which is the error of $\A(\matS)$ with respect to the population loss,  i.e., $F(\A(\matS)) - \min_\W F$. The test error is also sometimes called \textit{excess (population) loss}. 
Note that the expected test error can be decomposed into three terms:
\begin{align*}
& \E_{\A, \matS} \left [ F(\A(\matS)) \!-\! \min_\W F \right ] \quad \text{(test error)} = \\    
& \underbrace{\E_{\A, \matS} \left [ F(\A(\matS))\! -\! \hF(\A(\matS), \matS) \right ]}_{\text{generalization error}}
+ \underbrace{\E_{\A, \matS} \left [\hF(\A(\matS), \matS)\! -\! \min_\W \hF(., \matS) \right ]}_{\text{training error}}
+ \underbrace{\E_{\matS} \left [ \min_\W \hF(., \matS)\right ] \!-\! \min_\W F}_{\leq 0}.  
\end{align*}

It can be verified that the expectation of the third term (over $\A$ and $\matS$) is non-positive since 
$\E_{\matS}[\min_\W \hF(., \matS)] \leq \min_\W \E_{\matS}[\hF(., \matS)]$ and  
$\E_{\matS}[\hF(., \matS)] = F.$ 
Hence, to bound the expected test error, we should bound the expectation of training and generalization errors. 

The Model-Agnostic Meta-Learning (MAML) method proposed in \cite{finn17a} is designed to solve the empirical minimization problem defined in \eqref{eqn:empirical_prob}. The steps of MAML are outlined in Algorithm~\ref{Algorithm1}. MAML solves Problem~\eqref{eqn:empirical_prob} by using SGD update for the average loss function $\hF(w, \matS)$. To better highlight this point, note that the gradient of $\nabla \hF(w, \matS)$ can be written as $\frac{1}{m}\sum_{i=1}^m \nabla \hF_i(w, \matS_i) $, where the $i$-th term corresponding to task $\T_i$ is given by
\begin{align}\label{eqn:grad_empirical}
\nabla \hF_i(w, \matS_i) & = \frac{1}{\binom{n}{K}} \sum_{\substack{\D_i^\tein \subset \matS_i^\tein \\ |\D_i^\tein| = K}}
\Bigg [ (I_d -  \alpha \nabla^2 \hmatL(w,\D_i^\tein) )  
 \times \nabla \hmatL \left (w - \alpha \nabla \hmatL(w,\D_i^\tein), \matS_i^\teout \right) \Bigg ],
\end{align}
which involves the second-order information of the loss function. 
Therefore, to compute a mini-batch approximation for the above gradient, we consider the batches $\D_i^\tein \subset \matS_i^\tein$ with size $K$ and $\D_i^\teout \subset \matS_i^\teout$ with $b$ elements. Replacing the above sums with their batch approximations leads to the following stochastic gradient approximation 
\begin{align}\label{eqn:est_grad_empirical}
& g_i(w;\D_i^\tein, \D_i^\teout) := 
  (I_d -  \alpha \nabla^2 \hmatL(w,\D_i^\tein) ) \nabla \hmatL \left (w - \alpha \nabla \hmatL(w,\D_i^\tein), \D_i^\teout \right),	
\end{align}
which is indeed an unbiased estimator of the gradient $\nabla \hF_i(w, \matS_i)$ in \eqref{eqn:grad_empirical}. If for each task we perform the update of SGD with $g_i$ and then compute their average it would be similar to running SGD for the average loss $\nabla \hF(w, \matS)$. This is exactly how MAML is implemented in practice as outlined in 
Algorithm \ref{Algorithm1}. In this paper, we consider a constrained problem, and as a result, we also need an extra projection step in the last step to ensure the feasibility of iterates. Finally, the output of MAML could be the last iterate $w^T$ or the time-average of all iterates $\bar{w}^T := \frac{1}{T+1} \sum_{t=0}^{T} w^t$.

\begin{algorithm}[tb]
\caption{MAML \citep{finn17a}}
\label{Algorithm1}
	\begin{algorithmic}
    \STATE {\bfseries Input:} The set of datasets $\matS = \{\matS_i\}_{i=1}^m$ with $\matS_i = \{\matS_i^\tein, \matS_i^\teout\}$; test time batch size $K$; \# of tasks summoned at each round $r$; \# of iterations $T$. 
    \STATE Choose arbitrary initial point $w^0 \in \W$;
	\FOR{$t=0$ to $T-1$}
    	\STATE Choose $r$ tasks uniformly at random (out of $m$ tasks) and store their indices in $\B_t$;
    		\FOR{all $\T_i$ with $i \in \B_t$}
    		\STATE Sample a batch $\D_i^{t,\tein}$ of $K$ different elements from $\matS_i^\tein$ with replacement;
    		\STATE Sample a batch $\D_i^{t,\teout}$ of size $b$ from $\matS_i^\teout$ and with replacement;
    		\STATE $w^{t+1}_i :=  w^t - \beta_t \left (I_d -  \alpha \nabla^2 \hmatL(w^t,\D_i^{t,\tein}) \right) 
    		\nabla \hmatL \left (w^t - \alpha \nabla \hmatL(w^t,\D_i^{t,\tein}), \D_i^{t,\teout} \right)$;
    		\ENDFOR
        \STATE $w^{t+1} := r_\W \left( \frac{1}{r} \sum_{i \in \B_t} w^{t+1}_i \right) $;
    \ENDFOR
    \STATE  {\bfseries Return:} $w^T$ and $\bar{w}^T := \frac{1}{T+1} \sum_{t=0}^{T} w^t$ 
    \end{algorithmic}
\end{algorithm}

As stated earlier, the convergence properties of MAML-type methods from an optimization point of view have been studied recently under different set of assumptions. In this paper, as we characterize the sum of training error and generalization error, we briefly discuss the optimization error of MAML when it is used to solve the empirical problem in \eqref{eqn:empirical_prob}.
 However, the main focus of this paper is on studying the generalization error of MAML with respect to new samples and new tasks. Specifically, we aim to address the following questions: \textbf{(i)} How well does the solution of \eqref{eqn:empirical_prob} \textit{generalize} to the main problem of interest in \eqref{eqn:population_prob}? This could be seen as the \textit{generalization error} of the MAML algorithm over new samples for recurring tasks. \textbf{(ii)} How well does the solution of \eqref{eqn:empirical_prob} \textit{generalize} to samples from new unseen tasks? To be more precise, how would the obtained model preform if the new task is not one of the $m$ tasks $\T_1, \ldots, \T_m$ observed at training, and it is rather a new, \textit{unseen task} $\T_{m+1}$ with an unknown underlying distribution $p_{m+1}$? In the upcoming sections, we answer these questions on the generalization properties of MAML in detail and characterize the role of number of tasks $m$, number of samples per task $n$, and number of labeled samples revealed at test time $K$.
 
\section{Theoretical results}\label{sec:Theory}

In this section, we formally characterize the excess population loss (test error) of the MAML solution, when we measure the performance of a model after one step of SGD adaptation. In particular, we first discuss the training error of MAML in detail. Then, we establish a generalization error bound for the case that the solution of MAML is evaluated over new samples of a recurring task. Finally, we state the generalization error of MAML once its solution is applied to a new unseen task.  
Before stating our results, we mention our required assumptions.

\begin{assumption}\label{assumption_lipschitz}
For any $z \in \Z$, the function $\ell(.,z)$ is twice continuously differentiable. Furthermore, we assume it satisfies the following properties for any $w,u \in \R^d$:

\noindent\textbf{(i)} For any $z \in \Z$, the function $\ell(.,z)$ is $\mu$-strongly convex, i.e., $\Vert \nabla \ell(w,z) - \nabla \ell(u,z) \Vert \geq \mu \Vert w - u \Vert$;

\noindent\textbf{(ii)} The gradient norm is uniformly bounded by $G$ over $\W$, i.e., $\Vert \nabla \ell(w,z) \Vert \leq G$;

\noindent\textbf{(iii)} The loss is $L$-smooth over $\R^d$, i.e., $\Vert \nabla \ell(w,z) - \nabla \ell(u,z) \Vert \leq L \Vert w - u \Vert$;

\noindent\textbf{(iv)} Hessian is $\rho$-Lipschitz continuous over $\R^d$, i.e., $\Vert \nabla^2 \ell(w,z) - \nabla^2 \ell(u,z) \Vert \leq \rho \Vert w - u \Vert$.
\end{assumption}
We also require the following assumption on the tasks distribution. This assumption implies that, with probability one, a set of finite samples generated from a distribution $p_i$ are all different.
\begin{assumption}\label{assumption_non_atmoic}
We assume $\Z$ is a Polish space (i.e., complete, separable, and metric) and $\F_\Z$ is the Borel $\sigma$-algebra over $\Z$. Moreover, for any $i$, $p_i$ is a non-atomic probability distribution over $(\Z, \F_\Z)$, i.e., $p_i(z)=0$ for every $z \in \Z$.	
\end{assumption}  

\subsection{Training error}
While the main focus of this paper is on studying the population error of MAML algorithm, we first study its training error which is required  to provide characterization of the excess loss of MAML. To do so, we first state the following result from \citep{finn19a} and \citep{fallah2019convergence} on the strong convexity and smoothness of $\ell(w - \alpha \nabla \hmatL(w,\D), z)$ for any batch $\D$ and any $z \in \Z$. 
\begin{lemma}[\citep{fallah2019convergence} $\&$ \citep{finn19a}]
\label{lemma:F_smoothness_convexity}
If Assumption \ref{assumption_lipschitz} holds, then for an arbitrary batch $\D$ and $z \in \Z$, and with $\alpha \leq \tfrac{1}{L}$, the function $\ell (w - \alpha \nabla \hmatL(w,\D), z )$  is $4L+2\alpha \rho G$ smooth over $\W$. Furthermore,  $\ell (w - \alpha \nabla \hmatL(w,\D), z )$  is $\tfrac{\mu}{8}$-strongly convex, if $\alpha \leq \min\{\tfrac{1}{2L}, \tfrac{\mu}{8\rho G}\}$.	
\end{lemma}
An immediate consequence of this Lemma is that the MAML empirical loss $\hF$ defined in \eqref{eqn:empirical_prob} is also $\mu/8$-strongly convex and $4L+2\alpha \rho G$ smooth over $\W$. In addition, it can be shown that the norm of $g_i(w;\D_i^\tein, \D_i^\teout)$ defined in \eqref{eqn:est_grad_empirical}, which is the unbiased gradient estimate used in MAML, is uniformly bounded above; for more details check Lemma \ref{lemma:variance} in Appendix \ref{sec:intermediate_results}.
Having these properties of the MAML empirical loss established, we next state the following proposition on the training error of MAML. This result is obtained by
slightly modifying the well-known results on the convergence of SGD in \citep{rakhlin2011making, hazan2007logarithmic, nemirovski2009robust} in order to take into account the stepsize constraints that are imposed by generalization analysis. 
For completeness, the proof of this result is provided  in Appendix \ref{proof-proposition:main_training}.
\begin{proposition}
\label{proposition:main_training}
Consider $\hF(., \matS)$ defined in \eqref{eqn:empirical_prob} with $\alpha \leq \min\{\tfrac{1}{2L}, \tfrac{\mu}{8\rho G}\} $. If Assumption \ref{assumption_lipschitz} holds, then for MAML with $\beta_t = \min (\beta, \frac{8}{\mu (t+1)})$ for $\beta \leq 8/\mu$, and for any set $\matS$, the last iterate $w^T$ satisfies
\begin{align}\label{eqn:opt_last_iterate}
\E & \left [ \hF(w^T, \matS) - \hF(w^*_\matS, \matS) \right]  
\leq \bigO(1) \frac{G^2(1+\tfrac{1}{\beta \mu})}{\mu^2} \left ( \frac{L+ \rho \alpha G}{T} + \frac{G}{\sqrt{T}} \right ) , 
\end{align}
and the time-average of iterates $\bar{w}^T$ satisfies
\begin{align*}
\E & \left [ \hF(\bar{w}^T, \matS) - \hF(w^*_\matS, \matS) \right] \leq \bigO(1) \frac{G^2 (\log(T)+\tfrac{1}{\beta \mu})}{\mu T},
\end{align*}
where $w^*_\matS := \argmin_{w \in \W} \hF(., \matS)$ and the expectations are taken over the randomness of algorithm.
\end{proposition}
 
 In the above expressions, the notation $\bigO(1)$ only hides absolute constants. It is worth noting that the term $G/\sqrt{T}$ in \eqref{eqn:opt_last_iterate} vanishes, if $w^*_\matS$ be a minimizer of the unconstrained problem, i.e., $\nabla \hF(w^*_\matS, \matS) = 0$.
\subsection{Generalization error}
We derive our generalization bounds for MAML by establishing its algorithmic stability properties. The stability approach has been used widely to characterize the generalization properties for optimization algorithms such as stochastic gradient descent \citep{hardt2016train} or differentially private methods \citep{bassily2019private}. These arguments are based on showing the uniform stability of algorithms \citep{bousquet2002stability} which we restate it here.

\begin{definition}[\citep{bousquet2002stability}]\label{definition: uniform_stability}
Consider the problem of minimizing the empirical function $\hmatL(w, \matH)$ for some dataset $\matH$. A randomized algorithm $\A$ with output $w_\matH$  given dataset $\matH$ is called $\gamma$-uniformly stable if the following condition holds: Take the dataset $\tilde{\matH}$ which is the same as $\matH$, except at one data points. Then, we have 
$
\sup_{\tilde{z} \in \Z}\ \E_\A \left [ 
\left |\ell (w_\matH , \tilde{z})
- \ell (w_{\tilde{\matH}} , \tilde{z} ) \right |
\right ] 
\leq \gamma,	
$
where the expectation is taken over the randomness of $\A$.
\end{definition}
The above definition captures the stability of an algorithm. Specifically, it states that Algorithm $\A$ is $\gamma$-stable, if the resulting loss of its outputs, when it is run using to two different datasets that only differ in one data point, are at most $\gamma$ away from each other. Note that the above definition holds if the difference between the losses \textit{evaluated at any point $\tilde{z}$} is bounded by $\gamma$.
The main importance of this definition is its connection with generalization error. In particular, it can be shown that if an algorithm is $\gamma$-uniformly stable and ``symmetric", then its generalization error is bounded above by $\gamma$; see, e.g.,  \citep{bousquet2002stability}. Next, we formally state the definition of a symmetric algorithm.
\begin{definition}
An algorithm $\A: \Z^n \to \R^d$ is called symmetric, if for any $\matS \subset \Z^n$, the distribution of its output, i.e., $\A(\matS)$, does not depend on the ordering of elements of $\matS$, i.e., if we take $\matS'$ as a permutation of $\matS$, the distribution of $\A(\matS)$ and $\A(\matS')$ would be similar. 
\end{definition}

Note that Definition \ref{definition: uniform_stability} is useful for the case where we measure the performance of a model $w$ by its loss  function over a sample, i.e., $\ell (w , \tilde{z})$.
However, in this paper we measure the performance of a model by looking at its loss after one step of SGD which involves $K$ data points,  as defined in \eqref{eqn:est_grad_empirical}. Therefore, we cannot directly use Definition \ref{definition: uniform_stability} for characterizing the generalization error of MAML.  In fact, in what follows, we first propose a modified version of the uniform stability definition, which is compatible with our setting, and then show how such stability could lead to generalization bounds for MAML-type algorithms.

\begin{definition}\label{definition:K_uniform_stability}
Consider the problem in \eqref{eqn:empirical_prob}. A randomized algorithm $\A$ with output $w_\matS$ given dataset $\matS$ is called $(\gamma, K)$-uniformly stable if the following condition holds for any $i \in \{1,\ldots m\}$: Take the dataset $\tilde{\matS}$ which is the same as $\matS$, except that $\tilde{\matS}_i^\tein$ and $\tilde{\matS}_i^\teout$ differ from $\matS_i^\tein$ and $\matS_i^\teout$ in at most $K$ and one data points, respectively. Then, for any $\tilde{z} \in \Z$ and any $K$ distinct points $\{z_1,...,z_K\}$ in $\Z$, 
\begin{align*}
& \E_\A \left [ 
\left |\ell \left (w_\matS - \alpha \nabla \hmatL(w_\matS, \{z_j\}_{j=1}^K) , \tilde{z} \right) 
- \ell \left (w_{\tilde{\matS}} - \alpha \nabla \hmatL(w_{\tilde{\matS}}, \{z_j\}_{j=1}^K) , \tilde{z} \right) \right |
\right ] 
\leq \gamma,	
\end{align*}
where the expectation is taken over the randomness of $\A$.
\end{definition}
A few remarks about the above definition follow. First, one might wonder, why it is needed to change $K$ points of the set $\matS_i^\tein$, while we change only one point of the set $\matS_i^\teout$. Note that, going from \eqref{eqn:population_prob:b} to \eqref{eqn:empirical_prob:b}, the expectation $\E_{D_i^{\text{test}}}[.]$ is replaced by the sum over all $\binom{n}{K}$ possible batches $\D_i^\tein$ of size $K$ from $\matS_i^\tein$. In other words, for the empirical sum in~\eqref{eqn:empirical_prob:b}, each batch $\D_i^\tein$ can be seen as a data unit. That said, and similar to Definition \ref{definition: uniform_stability}, to characterize the stability, we need to change one data unit which is one batch of size $K$. That is why we change $K$ data points of $\matS_i^\tein$ in the definition of $(\gamma, K)$-uniformly stability. On the other hand, we replace $\matL_i(.) = \E_{z \sim p_i}[\ell(.,z)]$ in \eqref{eqn:population_prob:b} with a sum over $n$ points of $\matS_i^\teout$ in \eqref{eqn:empirical_prob:b}, and thus, for this one, each data unit is just a single data point. So, similar to Definition  \ref{definition: uniform_stability}, we just change one data point for the set $\matS_i^\teout$. 
 
Second, it is worth comparing this definition with the other definition given for stability of meta-learning algorithms in  \citep{chen2020closer}. In that paper, the definition of stability is based on modifying the \textit{whole dataset $\matS_i$} rather than what we do here which is changing just $K+1$ points. While taking such a definition makes the analysis relatively simpler, it prohibits us from characterizing the dependence of generalization error on $n$, and hence the resulting upper bound for generalization error would be larger. We will come back to this point later when we derive the stability of MAML with respect to Definition \ref{definition:K_uniform_stability} and compare it with the one obtained in \citep{chen2020closer}.

As we discussed, the main reason that we are interested in the uniform stability of an algorithm is its connection with generalization error. In the next theorem, we formalize this connection for MAML formulation and show that if an Algorithm $\A$ is $(\gamma, K)$-uniformly stable and symmetric, then its output generalization error is bounded above by $\gamma$. The proof of this result is available  in Appendix~\ref{proof-thm:stab_gen}.
\begin{theorem}\label{thm:stab_gen}
Consider the population and empirical losses defined in \eqref{eqn:population_prob} and \eqref{eqn:empirical_prob}, respectively. If Assumption \ref{assumption_non_atmoic} holds and $\mathcal{A}$ is a (possibly randomized) symmetric and $(\gamma, K)$-uniformly stable algorithm with output $w_\matS \in \W$, then  
$
\E_{\A, \matS} \left[ F(w_\matS) - \hF(w_\matS, \matS) \right] \leq \gamma.	
$	
\end{theorem}
This result shows that if we prove a symmetric algorithm is $(\gamma, K)$-uniformly stable as defined in Definition~\ref{definition:K_uniform_stability}, then we can bound its output model generalization error by $\gamma$. Hence, to characterize the generalization error of the model trained by MAML algorithm, we only need to capture the uniform stability parameter of MAML. 
Before stating this result, it is worth noting that while we limit our focus to MAML in this paper, Definition \ref{definition:K_uniform_stability} and Theorem \ref{thm:stab_gen} could provide a framework for studying the generalization properties of a broader class of gradient-based meta-learning algorithms such as Reptile \citep{nichol2018first},  First-order MAML \citep{finn17a}, and Hessian-Free MAML \citep{fallah2019convergence}.

\begin{theorem}\label{thm:stab_MAML}
If Assumption \ref{assumption_lipschitz} holds, then MAML (Algorithm \ref{Algorithm1}) with both last iterate and average iterate outputs and with $\alpha \leq \min\{\tfrac{1}{2L}, \tfrac{\mu}{8\rho G}\} $ and $\beta_t \leq \tfrac{1}{4L+2\alpha \rho G}$ is $(\gamma,K)$-uniformly stable, where
$
\gamma := \bigO(1) \frac{G^2 (1 + \alpha L K)}{mn \mu}.	
$	
\end{theorem}
According to the above discussion, the result of Theorem~\ref{thm:stab_MAML} guarantees that the generalization error of MAML solution decays by a factor of $\bigO(K/mn)$, where $m$ is the number of tasks in the training set and $n$ is the number of available samples per task.
The classic lower bound for SGD over strongly convex functions translates to a $\bigO(1/mn)$ lower bound in our setting. Hence, our bound is tight in the small $K$ regime, which is generally the case in few-shot learning problems. However, one shortcoming of this result is that it is not tight in the large $K$ regime. In Appendix \ref{app:largeK} we show how we could improve this result for the large $K$ regime. However, throughout the paper, we keep our discussion limited to the small $K$ regime.

\begin{remark}
If instead of using our uniform stability definition (i.e.,   Definition~\ref{definition:K_uniform_stability}), one uses the stability definition given in \citep{chen2020closer}, the resulted stability constant $\gamma$ would be proportional to $(1/m)$ rather than $(1/mn)$. In fact, our proposed uniform-stability definition   empowers us to obtain a better bound and indicates the role of number of samples per task $n$ in the  generalization error.  
\end{remark}
\begin{remark}
The algorithmic stability technique is mainly limited to the convex setting, since, in the nonconvex case, we need to keep learning rate very small to obtain meaningful generalization results which makes it impractical (Check Appendix \ref{sec:stability_limit} for further discussions on this matter).
In fact, the main reason that we assume $\ell$ is strongly convex and $\alpha \leq \Omega(\mu)$ is to ensure that the meta-objective is convex, as, in general, relaxing any of these two could lead to a nonconvex meta-objective function. However, these two assumptions together make the objective function strongly convex, which is not necessarily needed in our analysis. In fact, if we assume that $\ell$ and the meta-function are convex (but not necessarily strongly convex), we could still use Definition \ref{definition:K_uniform_stability} to derive similar generalization bounds.
\end{remark}

Putting Proposition \ref{proposition:main_training} and Theorem \ref{thm:stab_MAML} together, we obtain the following result on the excess population loss of MAML algorithm. We only report the result for the averaged iterates here, but one can obtain the result for the last iterate similarly by using Proposition \ref{proposition:main_training}.
\begin{proposition} \label{prop:excess_empirical}
Consider the function $F$ defined in \eqref{eqn:population_prob} with $\alpha \leq \min\{\tfrac{1}{2L}, \tfrac{\mu}{8\rho G}\} $. If Assumptions \ref{assumption_lipschitz} and \ref{assumption_non_atmoic} hold, then the average of iterates generated by MAML (Algorithm \ref{Algorithm1}) with $\beta_t = \min (\tfrac{1}{4L+2\alpha \rho G}, \tfrac{8}{\mu (t+1)})$ after $T$ iterations satisfies
\begin{align*}
& \E_{\A, \matS} \left [ F(\bar{w}^T) - \min_{\W} F \right] 
 \leq \bigO(1) \frac{G^2}{\mu} \left ( \frac{\log(T)+ L/\mu}{T} + \frac{1+\alpha L K}{mn} \right),
\end{align*}
where the expectation is taken over the sampling of $\matS$ and the randomness of MAML algorithm.
\end{proposition} 
As an immediate application, the following corollary characterizes MAML test error.
\begin{corollary}
Under the premise of Proposition \ref{prop:excess_empirical}, MAML algorithm after $T = \tilde{\bigO}(mn L/\mu)$ iterations returns $\bar{w}^T$ such that
$
\E_{\A, \matS} \left [ F(\bar{w}^T) - \min_{\W} F \right] \leq \bigO \left (G^2 (1+\alpha LK)/(mn \mu) \right).
$
\end{corollary}


\subsection{Generalization to an unseen task}
As we discussed in Section \ref{sec:Formulation}, another generalization measure is how the model trained with respect to the empirical problem in  \eqref{eqn:empirical_prob} performs on a new and unseen task $\T_{m+1}$ with corresponding distribution $p_{m+1}$. To state our result for this case, we first need to introduce the following distance notion between probability distributions.
\begin{definition}
For two distributions $P$ and $Q$, defined over the sample space $\Omega$ and $\sigma$-field $\F$, the total variation distance is defined as
$
\| P-Q \|_{TV} := \sup_{A \in \F} |P(A) - Q(A)|.	
$
\end{definition}
It is well-known that the total variation distance admits the following characterization
\begin{equation}\label{eqn:TV_distance_diff}
\| P-Q \|_{TV} = \sup_{f: 0 \leq f \leq 1} \E_{x \sim P}[f(x)] - \E_{x \sim Q}[f(x)].	
\end{equation}
Also, we require the following boundedness assumption for our result.
\begin{assumption}\label{assumption_boundedness}
For any $z \in \Z$, the function $\ell(.,z)$ is $M$-bounded over $\W$.
\end{assumption}
Considering these assumptions, we are ready to state our result for the case when the task at test time is a new task and is not observed during training. 
\begin{theorem}\label{thm:gen_new_task}
Consider the population losses defined in \eqref{eqn:population_prob:b} and \eqref{eqn:population_prob}. Suppose Assumptions \ref{assumption_lipschitz}, \ref{assumption_non_atmoic} and \ref{assumption_boundedness} hold. Then, for any $w \in \W$, we have  
\begin{align} \label{eqn:gen_new_task}
 &  |F_{m+1}(w) - F(w)| \leq D(p_{m+1}, \{p_i\}_{i=1}^m),
\end{align}
where
\begin{align}\label{eqn:D}
& D(p_{m+1}, \{p_i\}_{i=1}^m)  := \frac{4\alpha G^2}{m} \sum_{i=1}^m \|p_{m+1}-p_i\|_{TV} 
 + (M + 2 \alpha G^2) \| p_{m+1} - \frac{1}{m} \sum_{i=1}^m p_i \|_{TV}. 
\end{align} 	
\end{theorem}
While the proof is provided in detail in Appendix \ref{proof-thm:gen_new_task}, here we discuss a sketch of it to highlight the main technical contributions. To simplify the notation here, let us assume $m=1$, meaning that $p_1$ is the distribution used for training and $p_2$ is the distribution corresponding to the new task. 
Note that we aim to bound $|F_2(w) - F_1(w)|$. Recalling the definition of population loss \eqref{eqn:population_prob}, we need to bound the following expression (we drop the absolute value due to symmetry) 
\begin{equation} \label{eqn:sketch_1}
\begin{split}
\E_{\{z_j^{2} \sim p_{2} \}_{j=1}^K, \tilde{z}^2 \sim p_2} \! \big[ l \big (w - \alpha \nabla \hmatL(w, \{z_j^{2} \}_{j}), \tilde{z}^2 \big ) \! \big] \! \!
-\E_{\{z_j^1 \sim p_1 \}_{j=1}^K, \tilde{z}^1 \sim p_1} \! \big[ l \big (w - \alpha \nabla \hmatL(w, \{z_j^1\}_{j}), \tilde{z}^1 \big ) \! \big].	
\end{split}
\end{equation}	
Notice that this difference can be cast as
$
\E_{ (\{z_j \}_{j=1}^K, \tilde{z}) \sim p_2^{K+1}} [X] - \E_{ (\{z_j \}_{j=1}^K, \tilde{z}) \sim p_1^{K+1}} [X],	
$
with
$
X := l \left (w - \alpha \nabla \hmatL(w, \{z_j \}_{j=1}^K), \tilde{z} \right ).	
$
As a result, a naive approach would be using Lipschitz and boundedness properties of $l$ (Assumptions \ref{assumption_lipschitz} and \ref{assumption_boundedness}) along with \eqref{eqn:TV_distance_diff} to obtain a bound depending on $\| p_1^{K+1} - p_2^{K+1}\|_{TV} = \bigO(K) \|p_1 - p_2 \|_{TV}$. However, this bound is not tight as it grows with $K$. 

To address this issue, we exploit a coupling technique. Note that the expression in \eqref{eqn:sketch_1} does not depend on the joint distribution of $z_j^1$ and $z_j^2$, and instead, it only depends on the marginal distribution of $z_j^1$ and $z_j^2$. That said, for each $j$, we assume that $z_j^1$ and $z_j^2$ are sampled from a distribution $\mu$ on $\Z \times \Z$ such that 
$
z_j^1 \sim p_1$, $z_j^{2} \sim p_{2}$, and $\mu(z_j^1 \neq z_j^{2})	 = \| p_1 - p_{2} \|_{TV}.
$
Such a coupling exists and is called \textit{maximal coupling} of $p_1$ and $p_2$ \citep{den2012probability}. Using this idea, as we show in Appendix \ref{proof-thm:gen_new_task}, we can eliminate the dependence on $K$, and as a result, the upper bound in \eqref{eqn:gen_new_task} is independent of number of available labeled samples at test time denoted by $K$.
\sloppy
\begin{remark}
Note that the terms $\frac{1}{m} \sum_{i=1}^m \|p_{m+1}-p_i\|_{TV}$ and $\| p_{m+1} - \frac{1}{m} \sum_{i=1}^m p_i \|_{TV}$ in $D(p_{m+1}, \{p_i\}_{i=1}^m)$ come from the fact that we consider uniform distribution over tasks in the empirical problem \eqref{eqn:empirical_prob}. 
In particular, if we instead consider the empirical problem   
$
~ \argmin_{w \in \W} \sum_{i=1}^m q_i \hF_i(w, \matS_i),	
$
for some non-negative weights $q_i$ with $\sum_{i=1}^m q_i=1$, then $D(p_{m+1}, \{p_i\}_{i=1}^m)$ on the right hand side of \eqref{eqn:gen_new_task} would change to
\begin{align*}
& (M + 2 \alpha G^2) \| p_{m+1} - \sum_{i=1}^m q_i p_i \|_{TV}
+ 12\alpha G^2 \sum_{i=1}^m q_i \|p_{m+1}-p_i\|_{TV}.	
\end{align*}
This result shows that by changing the training problem we can achieve a lower generalization error for MAML, if we have some information about the distribution $p_{m+1}$ at  training time. For instance, if we know $p_{m+1}$ will be much closer to $p_1$ compared to $p_2$, making the weight of $p_1$ larger than $p_2$ would decrease the generalization error of MAML.
\end{remark}
\begin{corollary}\label{corr:excess_new_task}
Recall the population loss $F_{m+1}$ defined in~\eqref{eqn:population_prob} and $D(p_{m+1}, \{p_i\}_{i=1}^m)$ defined in Theorem \ref{thm:gen_new_task}. Let $\A$ be an algorithm for solving the empirical problem \eqref{eqn:empirical_prob} which achieves $\epsilon$ excess risk, i.e., $\E_{\A,\matS}[F(\A(\matS))] - \min_\W F \leq \epsilon$. If Assumptions \ref{assumption_lipschitz}, \ref{assumption_non_atmoic} and~\ref{assumption_boundedness} hold,  then algorithm $\A$ finds a model $w_\matS$ which achieves $\epsilon + D(p_{m+1}, \{p_i\}_{i=1}^m)$ excess loss with respect to $F_{m+1}$,
\begin{equation*}
 \E_{\A,\matS}[F_{m+1}(w_\matS)] - \min_\W F_{m+1} \leq \epsilon + 2D(p_{m+1}, \{p_i\}_{i=1}^m).
\end{equation*}
\end{corollary}
This corollary and Proposition \ref{prop:excess_empirical} together imply that the MAML algorithm's test error with respect to the new task $\T_{m+1}$ is $\bigO(1) \left( \frac{1}{mn} + D(p_{m+1}, \{p_i\}_{i=1}^m) \right)$. As a result, if the new task's distribution $p_{m+1}$ is sufficiently close to the other tasks' distributions, MAML will have a low test error on the new unseen task. On the other hand, if $p_{m+1}$ is far from $p_1, \ldots, p_m$ in TV distance, then test error of the model trained $\{\T_{i=1}^m\}$ over $\T_{m+1}$ could be potentially large. 
In Appendix \ref{proof-thm:gen_new_taskb} we show how this result can be extended to the case that the task at test time is generated from a distribution over both recurring tasks $\{\T_i\}_{i=1}^m$ and the unseen task $\T_{m+1}$.

\vspace{-1mm}
\section{Conclusion and future work} \label{sec:conclusion}
\vspace{-1mm}
In this work, we studied the generalization of MAML algorithm in two key cases: $a)$ when the test time task is a recurring task from the ones observed during the training stage, $b)$ when it is a new and unseen one. For the first one, and under strong convexity assumption, we showed that the generalization error improves as the number of tasks or the number of samples per task increases. For the second case, we showed that when the distance between the unseen task's distribution and the distributions of training tasks is sufficiently small, the MAML output generalizes well to the new task revealed at test time.

While we focused on the convex case in this paper, deriving generalization bounds when the meta-function is nonconvex is a natural future direction to explore. However, this could be challenging since the generalization of gradient methods is not well understood in the nonconvex setting even for the classic supervised learning problem.

\section{Acknowledgment}
Alireza Fallah acknowledges support from the Apple Scholars in AI/ML PhD fellowship and the MathWorks Engineering Fellowship.
This research is sponsored by the United States Air Force Research Laboratory and the United States Air Force Artificial Intelligence Accelerator and was accomplished under Cooperative Agreement Number FA8750-19-2-1000. The views and conclusions contained in this document are those of the authors and should not be interpreted as representing the official policies, either expressed or implied, of the United States Air Force or the U.S. Government. The U.S. Government is authorized to reproduce and distribute reprints for Government purposes notwithstanding any copyright notation herein.  
This research of Aryan Mokhtari is supported in part by NSF Grant 2007668, ARO Grant W911NF2110226, the Machine Learning Laboratory at UT Austin, and the NSF AI Institute for Foundations of Machine Learning.

\bibliographystyle{ieeetr}
\bibliography{main}

\begin{thebibliography}{10}

\bibitem{finn17a}
C.~Finn, P.~Abbeel, and S.~Levine, ``Model-agnostic meta-learning for fast
  adaptation of deep networks,'' in {\em Proceedings of the 34th International
  Conference on Machine Learning}, (Sydney, Australia), 06--11 Aug 2017.

\bibitem{Reptile}
A.~Nichol, J.~Achiam, and J.~Schulman, ``On first-order meta-learning
  algorithms,'' {\em arXiv preprint arXiv:1803.02999}, 2018.

\bibitem{khodak2019adaptive}
M.~Khodak, M.-F.~F. Balcan, and A.~S. Talwalkar, ``Adaptive gradient-based
  meta-learning methods,'' in {\em Advances in Neural Information Processing
  Systems}, pp.~5915--5926, 2019.

\bibitem{baker2016designing}
B.~Baker, O.~Gupta, N.~Naik, and R.~Raskar, ``Designing neural network
  architectures using reinforcement learning,'' in {\em International
  Conference on Learning Representations}, 2017.

\bibitem{zoph2016neural}
B.~Zoph and Q.~V. Le, ``Neural architecture search with reinforcement
  learning,'' in {\em International Conference on Learning Representations},
  2017.

\bibitem{zoph2018learning}
B.~Zoph, V.~Vasudevan, J.~Shlens, and Q.~V. Le, ``Learning transferable
  architectures for scalable image recognition,'' in {\em Proceedings of the
  IEEE conference on computer vision and pattern recognition}, pp.~8697--8710,
  2018.

\bibitem{ravi2016optimization}
S.~Ravi and H.~Larochelle, ``Optimization as a model for few-shot learning,''
  in {\em International Conference on Learning Representations}, 2017.

\bibitem{NIPS2016_Andry}
M.~Andrychowicz, M.~Denil, S.~G\'{o}mez, M.~W. Hoffman, D.~Pfau, T.~Schaul,
  B.~Shillingford, and N.~de~Freitas, ``Learning to learn by gradient descent
  by gradient descent,'' in {\em Advances in Neural Information Processing
  Systems 29}, pp.~3981--3989, Curran Associates, Inc., 2016.

\bibitem{MAML++}
A.~Antoniou, H.~Edwards, and A.~Storkey, ``How to train your {MAML},'' in {\em
  International Conference on Learning Representations}, 2019.

\bibitem{Meta-SGD}
Z.~Li, F.~Zhou, F.~Chen, and H.~Li, ``Meta-{SGD}: Learning to learn quickly for
  few-shot learning,'' {\em arXiv preprint arXiv:1707.09835}, 2017.

\bibitem{grant2018recasting}
E.~Grant, C.~Finn, S.~Levine, T.~Darrell, and T.~Griffiths, ``Recasting
  gradient-based meta-learning as hierarchical bayes,'' in {\em International
  Conference on Learning Representations}, 2018.

\bibitem{alpha-MAML}
H.~S. Behl, A.~G. Baydin, and P.~H.~S. Torr, ``Alpha {MAML:} adaptive
  model-agnostic meta-learning,'' 2019.

\bibitem{fallah2019convergence}
A.~Fallah, A.~Mokhtari, and A.~Ozdaglar, ``On the convergence theory of
  gradient-based model-agnostic meta-learning algorithms,'' in {\em
  International Conference on Artificial Intelligence and Statistics},
  pp.~1082--1092, 2020.

\bibitem{xu2020meta}
R.~Xu, L.~Chen, and A.~Karbasi, ``Meta learning in the continuous time limit,''
  {\em arXiv preprint arXiv:2006.10921}, 2020.

\bibitem{ji2020multi}
K.~Ji, J.~Yang, and Y.~Liang, ``Multi-step model-agnostic meta-learning:
  Convergence and improved algorithms,'' {\em arXiv preprint arXiv:2002.07836},
  2020.

\bibitem{wang2020global}
L.~Wang, Q.~Cai, Z.~Yang, and Z.~Wang, ``On the global optimality of
  model-agnostic meta-learning,'' in {\em International Conference on Machine
  Learning}, pp.~9837--9846, PMLR, 2020.

\bibitem{bousquet2002stability}
O.~Bousquet and A.~Elisseeff, ``Stability and generalization,'' {\em Journal of
  machine learning research}, vol.~2, no.~Mar, pp.~499--526, 2002.

\bibitem{hardt2016train}
M.~Hardt, B.~Recht, and Y.~Singer, ``Train faster, generalize better: Stability
  of stochastic gradient descent,'' in {\em International Conference on Machine
  Learning}, pp.~1225--1234, PMLR, 2016.

\bibitem{NEURIPS2019_072b030b}
A.~Rajeswaran, C.~Finn, S.~M. Kakade, and S.~Levine, ``Meta-learning with
  implicit gradients,'' in {\em Advances in Neural Information Processing
  Systems} (H.~Wallach, H.~Larochelle, A.~Beygelzimer, F.~d\textquotesingle
  Alch\'{e}-Buc, E.~Fox, and R.~Garnett, eds.), vol.~32, pp.~113--124, Curran
  Associates, Inc., 2019.

\bibitem{collins2020task}
L.~Collins, A.~Mokhtari, and S.~Shakkottai, ``Task-robust model-agnostic
  meta-learning,'' {\em Advances in Neural Information Processing Systems},
  vol.~33, 2020.

\bibitem{likhosherstov2020ufo}
V.~Likhosherstov, X.~Song, K.~Choromanski, J.~Davis, and A.~Weller, ``Ufo-blo:
  Unbiased first-order bilevel optimization,'' {\em arXiv preprint
  arXiv:2006.03631}, 2020.

\bibitem{chen2020solving}
T.~Chen, Y.~Sun, and W.~Yin, ``Solving stochastic compositional optimization is
  nearly as easy as solving stochastic optimization,'' {\em arXiv preprint
  arXiv:2008.10847}, 2020.

\bibitem{Hu2020BiasedSG}
Y.~Hu, S.~Zhang, X.~Chen, and N.~He, ``Biased stochastic gradient descent for
  conditional stochastic optimization,'' {\em ArXiv}, vol.~abs/2002.10790,
  2020.

\bibitem{finn19a}
C.~Finn, A.~Rajeswaran, S.~Kakade, and S.~Levine, ``Online meta-learning,'' in
  {\em Proceedings of the 36th International Conference on Machine Learning},
  vol.~97 of {\em Proceedings of Machine Learning Research}, (Long Beach,
  California, USA), pp.~1920--1930, PMLR, 09--15 Jun 2019.

\bibitem{fallah2020personalized}
A.~Fallah, A.~Mokhtari, and A.~Ozdaglar, ``Personalized federated learning with
  theoretical guarantees: A model-agnostic meta-learning approach,'' {\em
  Advances in Neural Information Processing Systems}, vol.~33, 2020.

\bibitem{liu2019taming}
H.~Liu, R.~Socher, and C.~Xiong, ``Taming maml: Efficient unbiased
  meta-reinforcement learning,'' in {\em International Conference on Machine
  Learning}, pp.~4061--4071, PMLR, 2019.

\bibitem{fallah2020provably}
A.~Fallah, K.~Georgiev, A.~Mokhtari, and A.~Ozdaglar, ``Provably convergent
  policy gradient methods for model-agnostic meta-reinforcement learning,''
  {\em arXiv preprint arXiv:2002.05135}, 2020.

\bibitem{chen2020closer}
J.~Chen, X.-M. Wu, Y.~Li, Q.~Li, L.-M. Zhan, and F.-l. Chung, ``A closer look
  at the training strategy for modern meta-learning,'' {\em Advances in Neural
  Information Processing Systems}, vol.~33, 2020.

\bibitem{guiroy2019towards}
S.~Guiroy, V.~Verma, and C.~Pal, ``Towards understanding generalization in
  gradient-based meta-learning,'' {\em arXiv preprint arXiv:1907.07287}, 2019.

\bibitem{rakhlin2011making}
A.~Rakhlin, O.~Shamir, and K.~Sridharan, ``Making gradient descent optimal for
  strongly convex stochastic optimization,'' {\em arXiv preprint
  arXiv:1109.5647}, 2011.

\bibitem{hazan2007logarithmic}
E.~Hazan, A.~Agarwal, and S.~Kale, ``Logarithmic regret algorithms for online
  convex optimization,'' {\em Machine Learning}, vol.~69, no.~2-3,
  pp.~169--192, 2007.

\bibitem{nemirovski2009robust}
A.~Nemirovski, A.~Juditsky, G.~Lan, and A.~Shapiro, ``Robust stochastic
  approximation approach to stochastic programming,'' {\em SIAM Journal on
  Optimization}, vol.~19, no.~4, pp.~1574--1609, 2009.

\bibitem{bassily2019private}
R.~Bassily, V.~Feldman, K.~Talwar, and A.~Guha~Thakurta, ``Private stochastic
  convex optimization with optimal rates,'' {\em Advances in Neural Information
  Processing Systems}, vol.~32, pp.~11282--11291, 2019.

\bibitem{nichol2018first}
A.~Nichol, J.~Achiam, and J.~Schulman, ``On first-order meta-learning
  algorithms,'' {\em arXiv preprint arXiv:1803.02999}, 2018.

\bibitem{den2012probability}
F.~Den~Hollander, ``Probability theory: The coupling method,'' {\em Lecture
  notes available online (http://websites. math. leidenuniv.
  nl/probability/lecturenotes/CouplingLectures. pdf)}, 2012.

\bibitem{nesterov_convex}
Y.~Nesterov, {\em Introductory Lectures on Convex Optimization: A Basic
  Course}, vol.~87.
\newblock Springer, 2004.

\end{thebibliography}
\newpage
\appendix
\begin{center}
\textbf{\LARGE{Appendix}}
\end{center}
\section{Intermediate Results} \label{sec:intermediate_results}
In this section we list a number of results that will be helpful in proofs of our main results. 
\begin{lemma}[From \cite{nesterov_convex} with modifications] 
\label{lemma:basic}
Let $\phi$ be a $\gamma$-strongly convex and $\eta$-smooth function which its gradient is bounded by $\tG$ over the convex and closed set $\W$. Then, we have
\begin{equation} \label{eqn:basic}
 \frac{\lambda}{2} \|w - w^*\|^2 \leq \phi(w) - \phi(w^*) \leq \frac{L}{2} \|w - w^*\|^2 + \tG \|w - w^*\|.
\end{equation}
\end{lemma}
\begin{proof}
Recalling the definition of strong convexity and smoothness, we have
\begin{equation} \label{eqn:basic_2}
 \frac{\lambda}{2} \|w - w^*\|^2 + \nabla \phi(w^*)^\top (w-w^*)
 \leq \phi(w) - \phi(w^*) 
 \leq \frac{L}{2} \|w - w^*\|^2 + \tG \|w - w^*\| + \nabla \phi(w^*)^\top (w-w^*).   
\end{equation}
Since $w^* = \argmin_\W \phi$, we have $\nabla \phi(w^*)^\top (w-w^*) \geq 0$, and hence from the left hand side of \eqref{eqn:basic_2}, we immediately obtain the left hand side of \eqref{eqn:basic}. To obtain the right hand side, it just suffices to use the bounded gradient assumption along with Cauchy–Schwarz inequality:
\begin{equation*}
\nabla \phi(w^*)^\top (w-w^*) \leq \tG \|w-w^*\|.    
\end{equation*}
\end{proof}
\begin{lemma}\label{lemma:ext_lipschitz}
Suppose the conditions in Assumption \ref{assumption_lipschitz} are satisfied. Then, with $\alpha \leq 1/L$, and for any batch $\D$ and $z \in \Z$, we have
\begin{equation}\label{eqn:ext_lipschitz}
\left \Vert \nabla \ell \left (w - \alpha \nabla \hmatL(w,\D), z \right) \right \Vert \leq 2G.
\end{equation} 
for any $w \in \W$. Furthermore, if we take $v \in \W$ as well, we have
\begin{equation}\label{eqn:ext_lipschitz_2}
\left | \ell \left (w - \alpha \nabla \hmatL(w,\D), z \right) 
- \ell \left (v - \alpha \nabla \hmatL(v,\D), z \right)
\right |
\leq 4G	\| w-v\|.
\end{equation}
\end{lemma}
\begin{proof}
First, note that
\begin{align}\label{eqn:intermediate_1}
\left \Vert \nabla \ell \left (w - \alpha \nabla \hmatL(w,\D), z \right) \right \Vert 
& \leq \| \nabla \ell(w,z) \| + \alpha L \| \hmatL(w,\D) \| \nonumber \\
& \leq (1+\alpha L)G \leq 2G,	
\end{align} 
where the first inequality follows from smoothness of $\ell(.,\tilde{z})$ for any $\tilde{z}$, and the second inequality is obtained using the bounded gradient assumption. 	
To show \eqref{eqn:ext_lipschitz_2}, let us define $\psi(w) = \ell \left (w - \alpha \nabla \hmatL(w,\D), z \right)$ for any $w \in \W$. Note that
\begin{align}
\psi(w) - \psi(v) &= \int_{0}^{1}	\nabla \psi(v + s (w-v))^\top (w-v) ds,
\end{align}
and hence, 
\begin{align}
& | \psi(w) - \psi(v)| \leq \int_{0}^{1}	 \| \nabla \psi(v + s (w-v))\| \cdot \| w-v \| ds \nonumber \\
&= \| w-v \| \int_{0}^{1}	 
\left \| 
\left ( I - \alpha \nabla^2 \hmatL(v + s (w-v),\D) \right ) \nabla \ell \left (v + s (w-v) - \alpha \nabla \hmatL(v + s (w-v),\D), z \right)
\right \|
ds \nonumber \\
& \leq 
2 \| w-v \| \int_{0}^{1}	 
\left \|  
\nabla \ell \left (v + s (w-v) - \alpha \nabla \hmatL(v + s (w-v),\D), z \right)
\right \|
ds,
\end{align}
where the last inequality follows from $\| \nabla^2 \hmatL(v + s (w-v),\D) \| \leq L$ and $\alpha \leq 1/L$.
Therefore, it suffices to bound 
\begin{align*}
\left \| 
\nabla \ell \left (v + s (w-v) - \alpha \nabla \hmatL(v + s (w-v),\D), z \right)
\right \|.
\end{align*}
Using the fact that $\W$ is convex, we have $v + s (w-v) \in \W$, and hence we could use the same approach in \eqref{eqn:intermediate_1} and complete the proof.
\end{proof}
\begin{lemma} \label{lemma:extended_boundedness}
Suppose Assumptions \ref{assumption_lipschitz} and \ref{assumption_boundedness} hold. Then, with $\alpha \leq 1/L$, and for any batch $\D$ and $z \in \Z$, we have
\begin{equation}\label{eqn:ext_lipschitz}
\left \Vert \ell \left (w - \alpha \nabla \hmatL(w,\D), z \right) \right \Vert \leq M + 2\alpha G^2,
\end{equation} 
for any $w \in \W$. 	
\end{lemma}
\begin{proof}
Let $h(\eta) := \ell \left (w - \eta \nabla \hmatL(w,\D), z \right)$. Using Lemma \ref{lemma:ext_lipschitz}, it is easy to verify that $|h'(\eta)| \leq 2G^2$, and hence, using Mean-value Theorem, we have $ |h(\alpha) - h(0) | \leq 2\alpha G^2$. This result, along with the fact that $|h(0)| = |\ell \left (w, z \right)| \leq M$ by Assumption \ref{assumption_boundedness} completes the proof. 
\end{proof}
As we stated in Section \ref{sec:Formulation}, MAML uses an unbiased gradient estimate at each iteration. The next lemma provides an upper bound on the variance of such estimate.
\begin{lemma}\label{lemma:variance}
Consider the function $\hF_i(., \matS_i)$ defined in \eqref{eqn:empirical_prob} with $\alpha \leq \tfrac{1}{L}$. Suppose the conditions in Assumption \ref{assumption_lipschitz} are satisfied. Recall that for 	batches $\D_i^\tein \subset \matS_i^\tein$ with size $K$ and $\D_i^\teout \subset \matS_i^\teout$ with size $b$, 
\begin{align*}
& g_i(w;\D_i^\tein, \D_i^\teout) =  
\left (I_d -  \alpha \nabla^2 \hmatL(w,\D_i^\tein) \right) \nabla \hmatL \left (w - \alpha \nabla \hmatL(w,\D_i^\tein), \D_i^\teout \right)	
\end{align*}
is an unbiased estimate of $\nabla \hF_i(w, \matS_i)$. Then, for any $w \in \W$, we have
\begin{subequations}
\begin{align*}
& \Vert g_i(w;\D_i^\tein, \D_i^\teout) \Vert \leq 4G, \\
& \E_{\D_i^\tein, \D_i^\teout} \left[
\left \| g_i(w;\D_i^\tein, \D_i^\teout) - \nabla \hF_i(w, \matS_i) \right \|^2
\right]	
\bigO(1) G^2 \left( \frac{\alpha^2 L^2}{K} + \frac{1}{b} \right ).
\end{align*}
\end{subequations}
\end{lemma}
\begin{proof}
Recall from Lemma \ref{lemma:ext_lipschitz} that
\begin{align}\label{eqn:bound_norm_1}
\left \Vert \nabla \ell \left (w - \alpha \nabla \hmatL(w,\D_i^\tein), z \right) \right \Vert 
\leq 2G
\end{align} 
As a result, we have
\begin{align*}
\Vert g_i(w;\D_i^\tein, \D_i^\teout) \Vert 
& \leq \| I_d -  \alpha \nabla^2 \hmatL(w,\D_i^\tein) \| \cdot \| \nabla \hmatL \left (w - \alpha \nabla \hmatL(w,\D_i^\tein), \D_i^\teout \right)	\|	\\
& \leq (1+\alpha L) 2G \leq 4G.
\end{align*}
To show the second result, we first claim
\begin{equation}
\E_{\D_i^\tein} \left[
\left \| g_i(w;\D_i^\tein, \matS_i^\teout) - g_i(w;\matS_i^\tein, \matS_i^\teout) \right \|^2
\right]	\leq 36 \frac{\alpha^2 L^2G^2}{K}.
\end{equation}
To show this, let us define
\begin{subequations}
\begin{align*}
e_H & := \left (I_d -  \alpha \nabla^2 \hmatL(w,\D_i^\tein) \right) - \left (I_d -  \alpha \nabla^2 \hmatL(w,\matS_i^\tein) \right) 
= \alpha \left ( \nabla^2 \hmatL(w,\matS_i^\tein) - \nabla^2 \hmatL(w,\D_i^\tein) \right) \\
e_G &:= \nabla \hmatL \left (w - \alpha \nabla \hmatL(w,\D_i^\tein), \matS_i^\teout \right)
- \nabla \hmatL \left (w - \alpha \nabla \hmatL(w,\matS_i^\tein), \matS_i^\teout \right).	 	
\end{align*}	
\end{subequations} 
Note that, by Assumption \ref{assumption_lipschitz}, we have
\begin{align}\label{eqn:bound_norm_2}
\|e_H\| \leq 2\alpha L,
\quad \|e_G\| \leq \alpha L \|\nabla \hmatL(w,\D_i^\tein) - \nabla \hmatL(w,\matS_i^\tein)\| \leq 2\alpha LG.	
\end{align}
In addition, using the fact that batch $\D_i^\tein$ is chosen uniformly at random, we have
\begin{align}\label{eqn:bound_norm_3}
\E_{\D_i^\tein} [\|e_H\|^2] & \leq \alpha^2 \frac{L^2}{K} \cdot \frac{n-K}{n-1},
\quad \E_{\D_i^\tein} [\|e_G\|^2] \leq \alpha^2 L^2 \frac{G^2}{K} \cdot \frac{n-K}{n-1}.
\end{align}
Next, note that 
\begin{align*}
& g_i(w;\D_i^\tein, \matS_i^\teout) - g_i(w;\matS_i^\tein, \matS_i^\teout) \\
&= e_H \nabla \hmatL \left (w - \alpha \nabla \hmatL(w,\matS_i^\tein), \matS_i^\teout \right)
+ e_G \left (I_d -  \alpha \nabla^2 \hmatL(w,\matS_i^\tein) \right)
+ e_G e_H.
\end{align*}
Hence, using Cauchy-Schwarz inequality along with \eqref{eqn:bound_norm_1}, we have
\begin{align*}
\E_{\D_i^\tein}& \left[
\left \| g_i(w;\D_i^\tein, \matS_i^\teout) - g_i(w;\matS_i^\tein, \matS_i^\teout) \right \|^2
\right] \\
&\leq 3 (2G)^2 \E_{\D_i^\tein} [\|e_H\|^2]
+ 3(1+\alpha L)^2 \E_{\D_i^\tein} [\|e_G\|^2]
+ 3 \E_{\D_i^\tein} [\|e_G e_H\|^2] \\
&\leq 12G^2 \E_{\D_i^\tein} [\|e_H\|^2] + 12 \E_{\D_i^\tein} [\|e_G\|^2] + 12\alpha^2 L^2 G^2 \E_{\D_i^\tein} [\|e_H\|^2].
\end{align*}
where the last inequality is obtained using \eqref{eqn:bound_norm_2} and $\alpha L \leq 1$. 
Now, using \eqref{eqn:bound_norm_2}, we have
\begin{align*}
\E_{\D_i^\tein}& \left[
\left \| g_i(w;\D_i^\tein, \matS_i^\teout) - g_i(w;\matS_i^\tein, \matS_i^\teout) \right \|^2
\right] \\
&\leq 12(2+\alpha^2 L^2) \frac{n-K}{n-1} \cdot \frac{\alpha^2 L^2G^2}{K} \leq 36 \frac{\alpha^2 L^2G^2}{K}.
\end{align*} 
which is the desired claim. Using this result and \eqref{eqn:bound_norm_1}, we imply
\begin{align}
\E_{\D_i^\tein, \D_i^\teout} & \left[
\left \| g_i(w;\D_i^\tein, \D_i^\teout) - \nabla \hF_i(w, \matS_i) \right \|^2
\right] \nonumber \\
& \leq \E_{\D_i^\tein} \left[
\left \| g_i(w;\D_i^\tein, \matS_i^\teout) - \nabla \hF_i(w, \matS_i) \right \|^2
\right] + \frac{4G^2}{b} \nonumber \\
& \leq 4 \left ( 36 \frac{\alpha^2 L^2G^2}{K} + \frac{G^2}{b} \right )
\end{align}
and the proof is complete.
\end{proof}
\section{Proof of Proposition \ref{proposition:main_training}} \label{proof-proposition:main_training}
Recall that 
\begin{equation*}
w^{t+1}  = \prod_\W \left ( w^t - \beta_t g^t \right),   
\end{equation*}
where $g^t := \frac{1}{r} \sum_{i \in \B_t} g_i(w^t ;\D_i^{t,\tein}, \D_i^{t,\teout})$ is an unbiased estimate of $\hF(w^t)$. Furthermore, by Lemma \ref{lemma:variance}, we know that $\|g^t\| \leq \tG:=4G$. Also, recall from Lemma \ref{lemma:F_smoothness_convexity} that $\hF$ is $\lambda$-strongly convex with  $\lambda := \mu/8$.

Let $\F^t$ be the $\sigma$-field generated by the information up to time $t$ (and not including iteration $t$, such as the randomness in $\B_t$, etc.) It is worth noting that $\E[g^t \mid \F^t] = \nabla \hF(w^t)$.

First, we claim that similar to the proof of Lemma 1 in \cite{rakhlin2011making}, we could show
\begin{equation} \label{eqn:recursive_1}
\E[\|w^{t+1} - w^* \|^2] \leq (1-2 \beta_t \lambda)\E[\|w^{t} - w^* \|^2] + \beta_t^2 \tG^2,  
\end{equation}
where $w^*$ is the minimizer of $\hF(.,\matS)$ over $\W$. To see this, and for the sake of completeness, let us recall the steps of the proof. Note that
\begin{align}
\E\left[\left\|w^{t+1}-w^*\right\|^{2}\right] &=
\E\left[\left\|\prod_{\W}\left(w^{t}-\beta_t g^t\right)-w^*\right\|^{2}\right] \nonumber \\
& \leq \E\left[\left\|w^{t}-\beta_t g^t-w^*\right\|^{2}\right]  \label{eqn:lemma_opt_1} \\
&=\E\left[\left\|w^{t}-w^*\right\|^{2}\right]-2 \beta_t \E\left[\left\langle g^t, w^{t}-w^*\right\rangle\right]+\beta_t^{2} \E\left[\left\|g^t\right\|^{2}\right] \nonumber \\
&=\E\left[\left\|w^{t}-w^*\right\|^{2}\right]-2 \beta_t \E\left[\left\langle \hF(w^t), w^{t}-w^*\right\rangle\right]+\beta_t^{2} \E\left[\left\| g^t\right\|^{2}\right]  \label{eqn:lemma_opt_2},
\end{align}
where \eqref{eqn:lemma_opt_1} follows from non-expansivity of projection and \eqref{eqn:lemma_opt_2} comes from the fact that $w^t \in \F^t$ and $\E[g^t \mid \F^t] = \nabla \hF(w^t)$. Now, having \eqref{eqn:lemma_opt_2}, and using $\|g^t\| \leq \tG$ along with the strong convexity of $\hF$, we have
\begin{align}
\E\left[\left\|w^{t+1}-w^*\right\|^{2}\right] & \leq 
\E\left[\left\|w^{t}-w^*\right\|^{2}\right]-2 \beta_t \E\left[\hF\left(w^{t}\right)-\hF\left(w^*\right)+\frac{\lambda}{2}\left\|w^{t}-w^*\right\|^{2}\right]+\beta_t^{2} \tG^{2} \nonumber \\
& \leq \E\left[\left\|w^{t}-w^*\right\|^{2}\right]-2 \beta_t \E\left[\frac{\lambda}{2}\left\|w^{t}-w^*\right\|^{2}+\frac{\lambda}{2}\left\|w^{t}-w^*\right\|^{2}\right]+\beta_t^{2} \tG^{2} \label{eqn:lemma_opt_3} \\
&=\left(1-2 \beta_t \lambda\right) \E\left[\left\|w^{t}-w^*\right\|^{2}\right]+\beta_t^{2} \tG^{2}, \nonumber
\end{align}
where \eqref{eqn:lemma_opt_3} follows from Lemma \ref{lemma:basic}.
Next, note that $\beta_t$ is given by
\[
    \beta_t = \left\{\begin{array}{lr}
        \beta, & \text{for } t \leq t^* - 1\\
        \frac{1}{\lambda (t+1)}, & \text{for } t > t^* -1
        \end{array} \right. ,
        \quad \text{with } t^* := \floor{\frac{1}{\beta \lambda}}.
  \]
For any $t \leq t^*$, from \eqref{eqn:recursive_1} and Lemma 2 in \cite{rakhlin2011making}, we obtain
\begin{align}
\E[\|w^{t} - w^* \|^2] \leq  \frac{\tG^2}{\lambda^2} + \beta^2 \tG^2 t \leq \frac{\tG^2 (t + 3)}{\lambda^2 (t+1)}.
\end{align}
Also, note that, for $t \geq t^*$, we have
\begin{equation}
\E[\|w^{t+1} - w^* \|^2] \leq (1-\frac{2}{t+1})\E[\|w^{t} - w^* \|^2] + \frac{\tG^2}{\lambda^2 (t+1)^2}.  
\end{equation}
Hence, by induction, it can be seen that for any t, we have
\begin{equation}
\E[\|w^{t} - w^* \|^2] \leq \frac{\tG^2 (t^* + 3)}{\lambda^2 (t+1)}.  
\end{equation}
Using Lemma \ref{lemma:basic} gives us \eqref{eqn:opt_last_iterate}.

To obtain the bound on the time-average iterate, first, we could similarly, modify the result in \cite{hazan2007logarithmic} to obtain
\begin{equation*}
2 \E [\hF(\bar{w}^T) - \hF(w^*)]
\leq \frac{1}{T} \left( 
\|w^0 - w^*\|^2 (\frac{1}{\beta_1} - \lambda)
+\sum_{t=1}^{T-1} \E[\|w^t-w^*\|^2] (\frac{1}{\beta_{t+1}} - \frac{1}{\beta_t} - \lambda)
+\tG^2 \sum_{t=1}^{T} \beta_t
\right).
\end{equation*}
It can be easily verified that for $\beta_t = \min (\beta, \frac{8}{\mu (t+1)})$, the term $\frac{1}{\beta_{t+1}} - \frac{1}{\beta_t} - \lambda$ is always non-positive. Hence, we have
\begin{align}
\E [\hF(\bar{w}^T) - \hF(w^*)]
& \leq \|w^0 - w^*\|^2 \frac{1/\beta - \lambda}{T+1} + \frac{2 \tG^2}{T+1} \sum_{t=0}^{T} \beta_t \nonumber \\
& \leq \bigO(1) \frac{\tG^2}{\lambda T} (1 + \log(T)-\log(t^*)) 
\leq \bigO(1) \frac{\tG^2}{\lambda T} \left (\frac{1}{\beta \lambda} + \log(T) \right),
\end{align}
where the last inequality follows from the fact that $4\tG^2/\lambda^2 \geq \|w^0 - w^*\|^2$ (see Lemma 2 in \cite{rakhlin2011making} for the proof.)
\section{Proof of Theorem \ref{thm:stab_gen}}\label{proof-thm:stab_gen}
To show the claim, it just suffices to show	that for any $i$, we have
\begin{align}\label{eqn:stab_gen_1}
\E_{\A, \matS} \left[ F_i(w_\matS) - \hF_i(w_\matS, \matS_i) \right] \leq \gamma.	
\end{align} 
Consider
\begin{equation*}
\matS_i^\tein = \{z^\tein_1,...,z^\tein_n\}, \quad \matS_i^\teout=\{z^\teout_1,...,z^\teout_n\}.	
\end{equation*}
To see this, first note that
\begin{align*}
F_i(w_\matS) 
= \E_{\{z_j\}_{j=1}^K , \tilde{z}} \left[ \ell \left (w_\matS - \alpha \nabla \hmatL(w_\matS, \{z_j\}_{j=1}^K), \tilde{z} \right ) \right],	
\end{align*}
where $\{z_j\}_{j=1}^K$ are $K$ distinct points sampled from $p_i$ and $\tilde{z}$ is also independently sampled from $p_i$. By Assumption \ref{assumption_non_atmoic}, we could assume $\tilde{z}$ is different from $K$ other points. Note that we have
\begin{align}\label{eqn:stab_gen_2}
\E_{\matS}[F_i(w_\matS)]	
= \E_{\matS, \{z_j\}_{j=1}^K , \tilde{z}} \left[ \ell \left (w_\matS - \alpha \nabla \hmatL(w_\matS, \{z_j\}_{j=1}^K), \tilde{z} \right ) \right].
\end{align}
Next, note that, we can write $\hF_i(w_\matS, \matS_i)$ as
\begin{align*}
\hF_i(w_\matS, \matS_i) = 
\frac{1}{\binom{n}{K} |\matS_i^\teout|} \sum_{\substack{\{\zeta_j\}_{j=1}^K \subset [n] \\ \tilde{\zeta} \in [n] }} 
\ell \left (w_\matS - \alpha \nabla \hmatL(w_\matS,\{z^\tein_{\zeta_j}\}_{j=1}^K), z^\teout_{\tilde{\zeta}} \right).	
\end{align*}
Thus, we have 
\begin{align*}
\E_{\A, \matS}[\hF_i(w_\matS, \matS_i)] = 
\frac{1}{\binom{n}{K} |\matS_i^\teout|} \sum_{\substack{\{\zeta_j\}_{j=1}^K \subset [n] \\ \tilde{\zeta} \in [n] }} 
\E_{\A, \matS} \left [\ell \left (w_\matS - \alpha \nabla \hmatL(w_\matS,\{z^\tein_{\zeta_j}\}_{j=1}^K), z^\teout_{\tilde{\zeta}} \right) \right ].	 	
\end{align*}
Notice that,  $\{\zeta_j\}_{j=1}^K$ are all different, and hence, due to the symmetry, all the expectations on the RHS are equal. Hence, for a fixed $\{\zeta_j\}_{j=1}^K \subset [n]$ and $\tilde{\zeta} \in [n]$, we have
\begin{align}
\E_{\A, \matS}[\hF_i(w_\matS, \matS_i)] &= 
\E_{\A, \matS}\left [\ell \left (w_\matS - \alpha \nabla \hmatL(w_\matS,\{z^\tein_{\zeta_j}\}_{j=1}^K), z^\teout_{\tilde{\zeta}} \right) \right ] \nonumber \\
&= \E_{\A, \matS, \{z_j\}_{j=1}^K , \tilde{z}} \left [\ell \left (w_\matS - \alpha \nabla \hmatL(w_\matS,\{z^\tein_{\zeta_j}\}_{j=1}^K), z^\teout_{\tilde{\zeta}} \right) \right ] \label{eqn:stab_gen_3}
\end{align} 
Next, define the dataset $\tilde{\matS}$ by substituting $z^\tein_{\zeta_j}$ with $z_j$, for all $j$, and $z^\teout_{\tilde{\zeta}}$ with $\tilde{z}$. It is straightforward to see that
\begin{align*}
\E_{\A, \matS, \{z_j\}_{j=1}^K , \tilde{z}} \left [\ell \left (w_\matS - \alpha \nabla \hmatL(w_\matS,\{z^\tein_{\zeta_j}\}_{j=1}^K), z^\teout_{\tilde{\zeta}} \right) \right ]
= \E_{\A, \matS, \{z_j\}_{j=1}^K , \tilde{z}} \left [\ell \left (w_{\tilde{\matS}} - \alpha \nabla \hmatL(w_{\tilde{\matS}},\{z_j\}_{j=1}^K), \tilde{z} \right) \right ]	
\end{align*}
Therefore, using \eqref{eqn:stab_gen_3}, we obtain
\begin{align}\label{eqn:stab_gen_4}
\E_{\A, \matS}[\hF_i(w_\matS, \matS_i)] 
= \E_{\A, \matS, \{z_j\}_{j=1}^K , \tilde{z}} \left [\ell \left (w_{\tilde{\matS}} - \alpha \nabla \hmatL(w_{\tilde{\matS}},\{z_j\}_{j=1}^K), \tilde{z} \right) \right ].	
\end{align} 
Putting \eqref{eqn:stab_gen_2} and \eqref{eqn:stab_gen_4} together, we have
\begin{align}
\E_{\A, \matS}& \left[ F_i(w_\matS) - \hF_i(w_\matS, \matS_i) \right] \nonumber\\
&\leq \E_{\A, \matS, \{z_j\}_{j=1}^K , \tilde{z}} \left[ 
\left | \ell \left (w_\matS - \alpha \nabla \hmatL(w_\matS, \{z_j\}_{j=1}^K), \tilde{z} \right ) 
- \ell \left (w_{\tilde{\matS}} - \alpha \nabla \hmatL(w_{\tilde{\matS}},\{z_j\}_{j=1}^K), \tilde{z} \right) \right |
\right] \nonumber \\
& = \E_{\matS, \{z_j\}_{j=1}^K , \tilde{z}} \left[ 
\E_\A \left[ 
\left |  \ell \left (w_\matS - \alpha \nabla \hmatL(w_\matS, \{z_j\}_{j=1}^K), \tilde{z} \right ) 
- \ell \left (w_{\tilde{\matS}} - \alpha \nabla \hmatL(w_{\tilde{\matS}},\{z_j\}_{j=1}^K), \tilde{z} \right) \right |
\right] \right]  	
\end{align}
where the last equality follows from Tonelli' theorem. Finally, note that since $\A$ is $(\gamma, K)$-uniformly stable, we could bound the the inner integral by $\gamma$, i.e., 
\begin{align*}
\E_\A \left[ 
\left |  \ell \left (w_\matS - \alpha \nabla \hmatL(w_\matS, \{z_j\}_{j=1}^K), \tilde{z} \right ) 
- \ell \left (w_{\tilde{\matS}} - \alpha \nabla \hmatL(w_{\tilde{\matS}},\{z_j\}_{j=1}^K), \tilde{z} \right) \right |
\right] \leq \gamma,	
\end{align*} 
and thus, we obtain the desired result \eqref{eqn:stab_gen_1}. 
\section{Proof of Theorem \ref{thm:stab_MAML}} \label{proof-thm:stab_MAML}
The stability definition says there is one $i$ such that the two datasets $\matS$ and $\tilde{\matS}$ differ only in the the two following terms:
 \begin{itemize}
 	\item  $\tilde{\matS}_i^\tein$ differs from $\matS_i^\tein$ in at most $K$ points. We show those $K$ samples by $\{z_j\}_{j=1}^K$ and $\{\tilde{z}_j\}_{j=1}^K$, respectively.
 	\item $\tilde{\matS}_i^\teout$  differs from $\matS_i^\teout$ in at most one point. We show those by $\zeta$ and $\tilde{\zeta}$, respectively.
 \end{itemize}
Let's consider two parallel processes of generating iterates $\{w^t\}$ and $\{\tilde{w}^t\}$ by using datasets $\matS$ and $\tilde{\matS}$, respectively. We use the tilde superscript to refer to the second process throughout the proof. Also, we use $D_i^{t, \teout}$ and $D_i^{t, \tein}$ to refer to indices of samples in $\D_i^{t, \teout}$ and $\D_i^{t, \tein}$, respectively. Also, with a slight abuse of notation, by $\hmatL(w^t,D_i^{\tein/\teout})$ we mean $\hmatL(w^t,\D_i^{\tein/\teout})$.

Note that the randomness of algorithm comes from the randomness in drawing batches at each iteration. We do a coupling argument here. We could assume the two parallel processes of generating iterates $\{w^t\}$ and $\{\tilde{w}^t\}$ use the same random machine for sampling batches. In other words, $\B_t = \tilde{\B}_t$, $D_i^{t, \teout} = \tilde{D}_i^{t, \teout}$, and $D_i^{t, \tein} = \tilde{D}_i^{t, \tein}$  

For one particular realization:
\begin{itemize}
	\item Let $u_t$ be the number of times that the index corresponding to sample $\zeta$ (or $\tilde{\zeta}$) is chosen in $D_i^{t, \teout}$. Note that this number could be zero if $i \notin \B_t$, and it could be greater than one if $i \in \B_t$ since $D_i^{t, \teout}$ is chosen with replacement.  
	\item Let $v_t$ be the number of indices corresponding to the samples $\{z_j\}_{j=1}^K$ (or $\{\tilde{z}_j\}_{j=1}^K$) that appears in $D_i^{t, \tein}$. Again, this number could be zero if $i \notin \B_t$. Also, note that we take $D_i^{t, \tein}$ as a batch of $K$ different samples from $\matS_i^\tein$, and hence, each one of $j$ indices appears at most one time in $D_i^{t, \tein}$.
\end{itemize}

The rest of the proof has three steps:
\begin{enumerate}
\item 
First, recall the definition of $b$ and $r$ from Alghorithm \ref{Algorithm1}. We claim
\begin{equation}\label{eqn:stab_bound1}
\E[u_t] = \frac{br}{nm}, \quad \E[v_t] = \frac{K^2r}{nm}.	
\end{equation}
The first one is easy to see. Task $i$ is in $\B_t$ with probability $r/m$, and if that happens, then $u_t$ would have a binomial distribution with mean $b/n$. To see the second one, note that
\begin{equation*}
\mathbb{P}(v_t = j) = \binom{K}{j} \binom{n-K}{K-j},	
\end{equation*}
and therefore,
\begin{align*}
\E[v_t|i \in \B_t] &= \frac{1}{\binom{n}{K}} \sum_{j=0}^K j \binom{K}{j} \binom{n-K}{K-j}. 	
\end{align*} 
Using the fact that $\binom{K}{j} = \tfrac{K}{j} \binom{K-1}{j-1}$, we obtain
\begin{align}
\E[v_t|i \in \B_t] &= \frac{K}{\binom{n}{K}} \sum_{j=0}^K \binom{K-1}{j-1} \binom{n-K}{K-j} \nonumber \\
 & = \frac{K}{\binom{n}{K}} \sum_{j=0}^{K-1} \binom{K-1}{j} \binom{(n-1)-(K-1)}{(K-1)-j}. \label{eqn:stab_bound2}
\end{align} 
However, note that $\binom{K-1}{j} \binom{(n-1)-(K-1)}{(K-1)-j}$ is exactly the probability of $v_t=j$ if $K \to K-1$ and $n \to n-1$. Hence, the sum $\sum_{j=0}^{K-1} \binom{K-1}{j} \binom{(n-1)-(K-1)}{(K-1)-j}$ is equal to $\binom{n-1}{K-1}$, and plugging this into \eqref{eqn:stab_bound2} gives us the second part of the claim \eqref{eqn:stab_bound1}.
 \item 
 Second, we claim that under Assumption \ref{assumption_lipschitz} we have
 \begin{align}\label{eqn:stab_bound3}
\E_\A [ \| w^{T} - \tilde{w}^{T} \|] 
& \leq 	\frac{4G}{mn} (1 + \alpha L K) \frac{16(2L+\rho \alpha G) + \mu}{\mu (2L+\rho \alpha G)}.
\end{align}
Before showing its proof, note that since $L \geq \mu$, this could be simplified as
 \begin{align}\label{eqn:stab_bound3'}
\E_\A [ \| w^{T} - \tilde{w}^{T} \|] 
& \leq 
\bigO(1) \frac{G}{mn \mu} (1 + \alpha L K).
\end{align}
Now, let's show why this is true.
To simplify the notation, let us define $\psi(w;\D,z):= \ell \left (w - \alpha \nabla \hmatL(w,\D), z \right)$. 
We start by revisiting the following lemma from \cite{hardt2016train}:
\begin{lemma}
Let $\phi$ be a $\lambda$-strongly convex and $\eta$-smooth function. Then, for any $\beta \leq \frac{2}{\lambda+\eta}$, we have
\begin{equation*}
\| (u - \beta \nabla \phi(u)) - (v - \beta \nabla \phi(v)) \| 
\leq (1-\frac{\beta \lambda \eta}{\lambda+\eta}) \| u-v\|,	
\end{equation*}	
for any $u$ and $v$.
\end{lemma}
Next, recall from Lemma \ref{lemma:F_smoothness_convexity} that for any batch $\D$ and any $z \in \Z$, $\psi(w;\D,z)$  is $4L+2\alpha \rho G$ smooth and $\mu/8$ strongly convex.
Hence, using the above lemma, for any $j \in \B_t$ that $j \neq i$, we have
\begin{equation}\label{eqn:stab_bound_N_1}
\| w_j^{t+1} - \tilde{w}_j^{t+1} \| \leq  \left( 1 - \beta_t \frac{2 \mu (2L+\rho \alpha G)}{16(2L+\rho \alpha G) + \mu} \right ) \| w^{t} - \tilde{w}^{t} \|.
\end{equation}
Next, let us assume $i \in \B_t$. 
In this case, we have
\begin{align}
\| w_i^{t+1} - \tilde{w}_i^{t+1} \| 
\leq 
& \frac{1}{b} \sum_{z \in \D_i^{t,\teout}} 
\left \| \left ( w^t - \beta_t \nabla \psi(w^t; \D^{t, \tein}_i,z) \right )
- \left( \tilde{w}^t - \beta_t \nabla \psi(\tilde{w}^t; \tilde{\D}^{t, \tein}_i,z) \right) 
 \right \|. \label{eqn:stab_bound_N_2a} \\
 & + \frac{1}{b} \beta_t \sum_{z \in \tilde{\D}_i^{t,\teout} / \D_i^{t,\teout}}
\left \|
\nabla \psi(\tilde{w}^t; \tilde{\D}^{t, \tein}_i,z) - \nabla \psi(w^t; \D^{t, \tein}_i,z) 
\right \|. \label{eqn:stab_bound_N_2b}
\end{align}
For \eqref{eqn:stab_bound_N_2b}, note that we know by Lemma \ref{lemma:variance} that $\Vert \nabla \psi(w;, \D,z) \Vert \leq 4G$, and hence, since $|\tilde{\D}_i^{t,\teout} / \D_i^{t,\teout}| = u_t$, we could bound the second term by $8 \beta_t G {u_t}/{b}$. As a result, we have
\begin{align}\label{eqn:stab_bound_N_2}
\| w_i^{t+1} - \tilde{w}_i^{t+1} \| 
& \leq 8\beta_t G \frac{u_t}{b} \nonumber \\
& + \frac{1}{b} \sum_{z \in \D_i^{t,\teout}} 
\left \| \left ( w^t - \beta_t \nabla \psi(w^t; \D^{t, \tein}_i,z) \right )
- \left( \tilde{w}^t - \beta_t \nabla \psi(\tilde{w}^t; \tilde{\D}^{t, \tein}_i,z) \right) 
 \right \|.
\end{align}
Note that
\begin{align} \label{eqn:stab_bound_N_3}
& \left \| \left ( w^t - \beta_t \nabla \psi(w^t; \D^{t, \tein}_i,z) \right )
- \left( \tilde{w}^t - \beta_t \nabla \psi(\tilde{w}^t; \tilde{\D}^{t, \tein}_i,z) \right) 
 \right \| \nonumber \\
& \leq 
\left \| \left ( w^t - \beta_t \nabla \psi(w^t; \D^{t, \tein}_i,z) \right )
- \left( \tilde{w}^t - \beta_t \nabla \psi(\tilde{w}^t; \D^{t, \tein}_i,z) \right) 
 \right \| \nonumber \\
 & + 
\beta_t \left \| \nabla \psi(\tilde{w}^t; \D^{t, \tein}_i,z)
- \nabla \psi(\tilde{w}^t; \tilde{\D}^{t, \tein}_i,z) \right \|.	
\end{align}
Let us bound the two terms on the RHS of \eqref{eqn:stab_bound_N_3} separately. First, similar to how we derived \ref{eqn:stab_bound_N_1}, we could bound the first term by
\begin{align} 
\label{eqn:stab_bound_N_4}
& \left \| \left ( w^t - \beta_t \nabla \psi(w^t; \D^{t, \tein}_i,z) \right )
- \left( \tilde{w}^t - \beta_t \nabla \psi(\tilde{w}^t; \D^{t, \tein}_i,z) \right) 
 \right \| \nonumber \\
& \leq \left( 1 - \beta_t \frac{2 \mu (2L+\rho \alpha G)}{16(2L+\rho \alpha G) + \mu} \right ) \| w^{t} - \tilde{w}^{t} \|.	
\end{align}
To bound the second term on the RHS of \eqref{eqn:stab_bound_N_3}, note that
\begin{align}
\label{eqn:stab_bound_N_5}
& \left \| \nabla \psi(\tilde{w}^t; \D^{t, \tein}_i,z) 
- \nabla \psi(\tilde{w}^t; \tilde{\D}^{t, \tein}_i,z) \right \| \nonumber \\
& = \left \|
(I - \alpha \nabla^2 \hmatL(\tilde{w}^t,\D^{t, \tein}_i) ) \nabla \ell \left (\tilde{w}^t - \alpha \nabla \hmatL(\tilde{w}^t,\D^{t, \tein}_i), z \right) \right. \nonumber \\
& \left. \qquad-  (I - \alpha \nabla^2 \hmatL(\tilde{w}^t, \tilde{\D}^{t, \tein}_i) ) \nabla \ell \left (\tilde{w}^t - \alpha \nabla \hmatL(\tilde{w}^t, \tilde{\D}^{t, \tein}_i), z \right)
\right \|	\nonumber \\
& \leq \left \|
\nabla \ell \left (\tilde{w}^t - \alpha \nabla \hmatL(\tilde{w}^t,\D^{t, \tein}_i), z \right)
- \nabla \ell \left (\tilde{w}^t - \alpha \nabla \hmatL(\tilde{w}^t, \tilde{\D}^{t, \tein}_i), z \right)
\right \| + \nonumber \\
& \alpha 
\left \|
\nabla^2 \hmatL(\tilde{w}^t,\D^{t, \tein}_i) \nabla \ell \left (\tilde{w}^t - \alpha \nabla \hmatL(\tilde{w}^t,\D^{t, \tein}_i), z \right)
- \nabla^2 \hmatL(\tilde{w}^t, \tilde{\D}^{t, \tein}_i) \nabla \ell \left (\tilde{w}^t - \alpha \nabla \hmatL(\tilde{w}^t, \tilde{\D}^{t, \tein}_i), z \right)
\right \| \nonumber \\
& \leq (1+ \alpha L) \left \|
\nabla \ell \left (\tilde{w}^t - \alpha \nabla \hmatL(\tilde{w}^t,\D^{t, \tein}_i), z \right)
- \nabla \ell \left (\tilde{w}^t - \alpha \nabla \hmatL(\tilde{w}^t, \tilde{\D}^{t, \tein}_i), z \right)
\right \| + \nonumber \\
& 2\alpha G \left \|
\nabla^2 \hmatL(\tilde{w}^t,\D^{t, \tein}_i) - \nabla^2 \hmatL(\tilde{w}^t, \tilde{\D}^{t, \tein}_i)
\right \|,
\end{align}
where, in the last inequality, we used Lemma \ref{lemma:ext_lipschitz} along with the third condition of Assumption \ref{assumption_lipschitz}. Hence, what remains is to bound the two terms in \eqref{eqn:stab_bound_N_5}. To do so, notice that
\begin{align}\label{eqn:stab_bound_N_6}
& \left \|
\nabla \ell \left (\tilde{w}^t - \alpha \nabla \hmatL(\tilde{w}^t,\D^{t, \tein}_i), z \right)
- \nabla \ell \left (\tilde{w}^t - \alpha \nabla \hmatL(\tilde{w}^t, \tilde{\D}^{t, \tein}_i), z \right)
\right \| \nonumber \\
& \leq \alpha L \left \|
\nabla \hmatL(\tilde{w}^t,\D^{t, \tein}_i) - \nabla \hmatL(\tilde{w}^t, \tilde{\D}^{t, \tein}_i)
\right \|	
\leq 2\alpha L G \frac{v_t}{K},
\end{align}
and 
\begin{align}\label{eqn:stab_bound_N_7}
\left \|
\nabla^2 \hmatL(\tilde{w}^t,\D^{t, \tein}_i) - \nabla^2 \hmatL(\tilde{w}^t, \tilde{\D}^{t, \tein}_i)
\right \| 
\leq 2L \frac{v_t}{K}.	
\end{align}
By plugging \eqref{eqn:stab_bound_N_6} and \eqref{eqn:stab_bound_N_7} into \eqref{eqn:stab_bound_N_5} and using $\alpha L \leq 1$, we have
\begin{equation}\label{eqn:stab_bound_N_8}
\left \| \nabla \psi(\tilde{w}^t; \D^{t, \tein}_i,z) 
- \nabla \psi(\tilde{w}^t; \tilde{\D}^{t, \tein}_i,z) \right \|
\leq 8 \alpha L G \frac{v_t}{K}.	
\end{equation}
Substituting this bound and \eqref{eqn:stab_bound_N_4} into \eqref{eqn:stab_bound_N_3} and plugging the result into \eqref{eqn:stab_bound_N_2}, we have
\begin{equation}\label{eqn:stab_bound_N_9}
\| w_i^{t+1} - \tilde{w}_i^{t+1} \| 
\leq  \left( 1 - \beta_t \frac{2 \mu (2L+\rho \alpha G)}{16(2L+\rho \alpha G) + \mu} \right ) \| w^{t} - \tilde{w}^{t} \|
+ 8 \beta_t G (\frac{u_t}{b} + \alpha L \frac{v_t}{K}).	
\end{equation}
Using \eqref{eqn:stab_bound_N_9} and \eqref{eqn:stab_bound_N_1}, we obtain
\begin{align*}
&\| \frac{1}{r} \sum_{j \in \B_t} w_j^{t+1} - \frac{1}{r} \sum_{j \in \B_t} \tilde{w}_j^{t+1} \| 
\leq  \left( 1 - \beta_t \frac{2 \mu (2L+\rho \alpha G)}{16(2L+\rho \alpha G) + \mu} \right ) \| w^{t} - \tilde{w}^{t} \|
+ 8 \beta_t G (\frac{u_t}{rb} + \alpha L \frac{v_t}{rK}).   
\end{align*}
Since projections are non-expansive, we have 
\begin{equation}\label{eqn:stab_bound_N_10}
\| w^{t+1} - \tilde{w}^{t+1} \| 
\leq  \left( 1 - \beta_t \frac{2 \mu (2L+\rho \alpha G)}{16(2L+\rho \alpha G) + \mu} \right ) \| w^{t} - \tilde{w}^{t} \|
+ 8 \beta_t G (\frac{u_t}{rb} + \alpha L \frac{v_t}{rK}).	
\end{equation}
Taking an expectation from both sides and using \eqref{eqn:stab_bound1}, we get
\begin{equation}\label{eqn:stab_bound_N_11}
\E_\A [\| w^{t+1} - \tilde{w}^{t+1} \|] 
\leq  \left( 1 - \beta_t \frac{2 \mu (2L+\rho \alpha G)}{16(2L+\rho \alpha G) + \mu} \right ) \E_\A [\| w^{t} - \tilde{w}^{t} \|]
+ 8 \frac{\beta_t G}{mn} (1 + \alpha L K).	
\end{equation}
Note that we can rewrite this bound as
\begin{equation*}
\E_\A [\| w^{t+1} - \tilde{w}^{t+1} \|] 
\leq  ( 1 - \beta_t \lambda ) \E_\A [\| w^{t} - \tilde{w}^{t} \|]
+ \beta_t \eta,
\end{equation*}
where
\begin{align*}
\lambda:= \frac{2 \mu (2L+\rho \alpha G)}{16(2L+\rho \alpha G) + \mu}, 
\quad  \eta:= \frac{8G}{mn} (1 + \alpha L K). 
\end{align*}
Note that the claim \eqref{eqn:stab_bound3} is in fact to show
\begin{equation*}
\E_\A [\| w^{t} - \tilde{w}^{t} \|] \leq \frac{\eta}{\lambda}.    
\end{equation*}
This is true for $t=1$ since $\beta_0 \leq \frac{1}{4L+2\rho \alpha G} \leq \frac{1}{\lambda}$. Having this, we could easily obtain the result by induction. 
 \item
 We are ready to conclude. Note that by Lemma \ref{lemma:ext_lipschitz}, we have 
 \begin{align*}
& \left | 
\ell \left (w^T - \alpha \nabla \hmatL(w^T, \{z_j\}_{j=1}^K) , \tilde{z} \right) 
- \ell \left (\tilde{w}^T - \alpha \nabla \hmatL(\tilde{w}^T, \{z_j\}_{j=1}^K) , \tilde{z} \right)
\right | \nonumber \\
& \leq 4G 
\left \|
\left (w^T - \alpha \nabla \hmatL(w^T, \{z_j\}_{j=1}^K) , \tilde{z} \right) 
- \left (\tilde{w}^T - \alpha \nabla \hmatL(\tilde{w}^T, \{z_j\}_{j=1}^K) , \tilde{z} \right)
\right \| \nonumber \\
& \leq 4\psi(1+\alpha L) \| w^T - \tilde{w}^T\| 
\leq 8G \| w^T - \tilde{w}^T\| .
\end{align*}
Taking expectations from both sides completes the proof for $w^T$. Note that \eqref{eqn:stab_bound3'} can be extended to $\bar{w}^T$ as well, and using an argument similar to this step, we could show the same stability bound for the average itrtaes as well.
\end{enumerate} 
\section{Generalization bound for large $K$ regime} \label{app:largeK}
Under the premise of Theorem \ref{thm:stab_MAML}, we claim
\begin{equation}
\E_{\A, \matS} \left[ F(w_\matS) - \hF(w_\matS, \matS) \right] \leq 
\bigO(1) G^2 \left ( \frac{1}{mn \mu} + \alpha \min \left \{ \frac{LK}{mn \mu}, \frac{1}{\sqrt{K}} \right \} \right).
\end{equation}
To show this, first, recall that
\begin{equation*}
F_i(w) = \E_{\D_i^{\text{test}}} \left[ \matL_i \left (w - \alpha \nabla \hmatL(w, \D_i^{\text{test}}) \right ) \right].    
\end{equation*}
Let $G_i(w) := \matL_i \left (w - \alpha \nabla \matL_i(w) \right ).$ Note that
\begin{align*}
| F_i(w) - G_i(w) | & = 
\left | \E_{\D_i^{\text{test}}} \left[ \matL_i \left (w - \alpha \nabla \hmatL(w, \D_i^{\text{test}}) \right ) 
-  \matL_i \left (w - \alpha \nabla \matL_i(w) \right ) \right]   \right | \\
& \leq 4\alpha G \E_{\D_i^{\text{test}}} \left | \hmatL(w, \D_i^{\text{test}}) - \nabla \matL_i(w) \right | \leq 4 \alpha \frac{G^2}{\sqrt{K}}.
\end{align*}
As a result, for $G(w) = \frac{1}{m} \sum_{i=1}^m G_i(w)$, we have
\begin{equation*}
|G(w) - F(W)| \leq \bigO(1) \alpha \frac{G^2}{\sqrt{K}}.    
\end{equation*}
Similarly, if we define
\begin{equation*}
\hat{G}_i(w) := \hmatL \left (w - \alpha \nabla \matL_i(w), \matS_i^\teout \right), \quad \hat{G}_w := \frac{1}{m} \sum_{i=1}^m G_i(w),  
\end{equation*}
we could show that
\begin{equation*}
\E_{\matS} \left | \hat{G}(w) - \hF(w, \matS) \right | \leq \bigO(1) \alpha \frac{G^2}{\sqrt{K}}.    
\end{equation*}
Finally, note that the well-known generalization results for strongly convex functions by using classic stability definition (Definition \ref{definition: uniform_stability}) implies (see \cite{hardt2016train} for details)
\begin{equation*}
\E_{\A, \matS} |G(w_{\A}) - \hat{G}(w_{\A})| \leq \bigO(1) \frac{G^2}{mn \mu}, 
\end{equation*}
where $w_{\A}$ is MAML output. Putting these bounds together, we obtain $\bigO(1) \left ( \frac{G^2}{mn \mu} + \alpha \frac{G^2}{\sqrt{K}} \right )$. Taking minimum of this and Theorem \ref{thm:stab_MAML} proves the aforementioned claim.

Finally, it is worth mentioning that while we are not sure whether our bound is tight for the large $K$ regime, this is not necessarily the case that the generalization bound improves as $K$ increases. To see this, consider MAML with only one task, i.e., $m=1$, and the quadratic loss $l(w,z) = (w^\top x - y)^2$ with $z=(x,y)$. In addition, and to focus on the generalization error coming from test update, we assume we have access to exact gradients for outer loop, i.e., 
$$\hat{F}(w)= \frac{1}{\binom{n}{K}} \sum_{\{z_i\} \subset \mathcal{D}^{in}} \mathcal{L} \left (w - \alpha \sum_{i=1}^K \frac{1}{K} \nabla l(w,z_i) \right )$$
Let $\Lambda = \mathbb{E}[xx^\top]$ and $\rho = \mathbb{E}[xy]$. Also, we denote the estimation of $\Lambda$ and $\rho$ over $\mathcal{D}^{in}$ by $\hat{\Lambda}$ and $\hat{\rho}$, respectively.

After some simplifications, it can be shown that
$$\nabla F(w) = \Lambda w - \rho - 2 \alpha \Lambda^2 w + 2 \alpha \Lambda \rho + \mathcal{O}(\alpha^2) ,$$
$$\nabla \hat{F}(w) = \Lambda w - \rho + 2 \alpha \Lambda \hat{\Lambda} w + \alpha \Lambda \hat{\rho} + \alpha \hat{\Lambda} \rho + \mathcal{O}(\alpha^2).$$
It can be seen that the difference of the two gradients is $\Omega(\frac{\alpha}{n})$ and does not decrease as $K$ increases. 
\section{Proof of Theorem \ref{thm:gen_new_task}} \label{proof-thm:gen_new_task}
First, we show the following lemma:
\begin{lemma} \label{lemma:coupling}
For any $\tilde{z}$ and any $w \in \W$, we have
\begin{align} \label{eqn:gen_new_task_6}
& \left \vert \E_{\{z_j^{m+1} \sim p_{m+1} \}_{j=1}^K} \left[ \ell \left (w - \alpha \nabla \hmatL(w, \{z_j^{m+1} \}_{j=1}^K), \tilde{z} \right ) \right]
-\E_{\{z_j^i \sim p_i \}_{j=1}^K} \left[ \ell \left (w - \alpha \nabla \hmatL(w, \{z_j^i\}_{j=1}^K), \tilde{z} \right ) \right] \right \vert	\nonumber \\
& \leq 4 \alpha G^2 \| p_{m+1} - p_i\|_{TV}.
\end{align}	
\end{lemma}
\begin{proof}
Note that since $p_i$ are non-atmoic, we could assume $z_j^i$'s are drawn independently. Same story holds for $z_j^{m+1}$'s. Now, for any $j$, let us assume $(z_j^i, z_j^{m+1})$ is drawn from a joint distribution of $p_i$ and $p_{m+1}$ corresponding to the maximal coupling of these distributions, i.e.,
\begin{equation*}
z_j^i \sim p_i, \quad z_j^{m+1} \sim p_{m+1}, \quad \mathbb{P}(z_j^i \neq z_j^{m+1})	 = \| p_i - p_{m+1} \|_{TV}.
\end{equation*}  
Hence, with probability $\binom{K}{t} (\| p_i - p_{m+1} \|_{TV})^t (1-\| p_i - p_{m+1} \|_{TV})^{K-t}$, we have $z_j^i \neq z_j^{m+1}$ for $t$ choices of $j$ (out of $1,...,K$). 

In addition, similar to the proof of Lemma \ref{lemma:extended_boundedness}, we could show that 
\begin{align*}
&\left \| 
\ell \left (w - \alpha \nabla \hmatL(w, \{z_j^{m+1} \}_{j=1}^K), \tilde{z} \right )
- 
\ell \left (w - \alpha \nabla \hmatL(w, \{z_j^{i} \}_{j=1}^K), \tilde{z} \right )
\right \|	\\
& \leq
2 \alpha G  
\| \hmatL(w, \{z_j^{m+1} \}_{j=1}^K) - \nabla \hmatL(w, \{z_j^{i} \}_{j=1}^K)\|.
\end{align*}
Hence, if $z_j^i \neq z_j^{m+1}$ for $t$ choices of $j$, then we have	
\begin{align*}
&\left \| 
\ell \left (w - \alpha \nabla \hmatL(w, \{z_j^{m+1} \}_{j=1}^K), \tilde{z} \right )
- 
\ell \left (w - \alpha \nabla \hmatL(w, \{z_j^{i} \}_{j=1}^K), \tilde{z} \right )
\right \|	\\
& \leq
4 \alpha G^2 \frac{t}{K}.  
\end{align*}
As a result, we have
\begin{align}
&\E_{\{z_j^{m+1} \sim p_{m+1} \}_{j=1}^K} \left[ \ell \left (w - \alpha \nabla \hmatL(w, \{z_j^{m+1} \}_{j=1}^K), \tilde{z} \right ) \right]
-\E_{\{z_j^i \sim p_i \}_{j=1}^K} \left[ \ell \left (w - \alpha \nabla \hmatL(w, \{z_j^i\}_{j=1}^K), \tilde{z} \right ) \right]	\nonumber \\
& \leq \sum_{t=0}^K \binom{K}{t} (\| p_i - p_{m+1} \|_{TV})^t (1-\| p_i - p_{m+1} \|_{TV})^{K-t} \cdot 4 \alpha G^2 \frac{t}{K} \nonumber \\
& = 4 \alpha G^2 (\| p_i - p_{m+1} \|_{TV}) \sum_{t=0}^K \frac{t}{K} \binom{K}{t} (\| p_i - p_{m+1} \|_{TV})^{t-1} (1-\| p_i - p_{m+1} \|_{TV})^{K-t} \nonumber \\
& = 4 \alpha G^2 (\| p_i - p_{m+1} \|_{TV}),	 \nonumber \\
\end{align}
where the last equality follows from the fact that
\begin{equation*}
\frac{t}{K} \binom{K}{t} (\| p_i - p_{m+1} \|_{TV})^{t-1} (1-\| p_i - p_{m+1} \|_{TV})^{K-t} 
= \binom{K-1}{t-1} (\| p_i - p_{m+1} \|_{TV})^{t-1} (1-\| p_i - p_{m+1} \|_{TV})^{K-1-(t-1)}.
\end{equation*}
\end{proof}
Let's get back to the proof of Theorem \ref{thm:gen_new_task}. For any $1 \leq i \leq m+1$ and any $\tilde{z}$, let us define
\begin{equation*}
X_i(\tilde{z})	:= \E_{\{z_j \sim p_i \}_{j=1}^K} \left[ \ell \left (w_\matS - \alpha \nabla \hmatL(w_\matS, \{z_j\}_{j=1}^K), \tilde{z} \right ) \right].
\end{equation*}
In other words, $X_i$ is the loss over data point $\tilde{z}$ when the model is updated using the distribution of task $i$.
Next, note that
\begin{align}\label{eqn:gen_new_task_3}
& F_{m+1}(w)	 - F_i(w) =	 \\
& 
\E_{\{z_j \sim p_{m+1} \}_{j=1}^K , \tilde{z} \sim p_{m+1}} \left[ \ell \left (w - \alpha \nabla \hmatL(w, \{z_j\}_{j=1}^K), \tilde{z} \right ) \right] 
- \E_{\{z_j \sim p_i \}_{j=1}^K , \tilde{z} \sim p_i} \left[ \ell \left (w - \alpha \nabla \hmatL(w, \{z_j\}_{j=1}^K), \tilde{z} \right ) \right]. \nonumber
\end{align}
Note that by Lemma \ref{lemma:extended_boundedness}, the term inside expectation is bounded, and hence, by Fubini's theorem, we can cast this term as
\begin{align}\label{eqn:gen_new_task_4}
\E_{\tilde{z} \sim p_{m+1}} [X_{m+1}(\tilde{z})] - \E_{\tilde{z} \sim p_{i}} [X_i(\tilde{z})]
\end{align}
By Lemma \ref{lemma:coupling}, we have $|X_i(\tilde{z}) - X_{m+1}(\tilde{z})| \leq 4\alpha G^2 \| p_i - p_{m+1}\|_{TV}$. Hence, we have
\begin{align}\label{eqn:gen_new_task_7}
&  F_{m+1}(w) - F_i(w) = \E_{\tilde{z} \sim p_{m+1}} [X_{m+1}(\tilde{z})] - \E_{\tilde{z} \sim p_{i}} [X_i(\tilde{z})] 
= 
\E_{\tilde{z} \sim p_{m+1}} [X_{m+1}(\tilde{z})] - \E_{\tilde{z} \sim p_{i}} [X_{m+1}(\tilde{z})]  
+ e_{i,m}, 	\nonumber 
\end{align}
where $|e_{i,m}| \leq 4\alpha G^2 \| p_i - p_{m+1}\|_{TV}$.
As a result, we have
\begin{align}
\left \vert F_{m+1}(w) - \frac{1}{m} \sum_{i=1}^m F_i(w) \right \vert 
& \leq 
\left \vert \E_{\tilde{z} \sim p_{m+1}} [X_{m+1}(\tilde{z})] - \frac{1}{m} \sum_{i=1}^m  \E_{\tilde{z} \sim p_{i}} [X_{m+1}(\tilde{z})] \right \vert
+ 4\alpha G^2 \| p_i - p_{m+1}\|_{TV}
\end{align}
Using Lemma \ref{lemma:extended_boundedness}, we have $0 \leq X_{m+1}(\tilde{z}) \leq M+2\alpha G^2$. Hence, by \eqref{eqn:TV_distance_diff}, we have
\begin{equation}\label{eqn:gen_new_task_8}
\left \vert
\E_{\tilde{z} \sim p_{m+1}} [X_{m+1}(\tilde{z})] - \frac{1}{m} \sum_{i=1}^m \E_{\tilde{z} \sim p_{i}} [X_{m+1}(\tilde{z})] \right \vert
\leq (M+2\alpha G^2) \| p_{m+1} - \frac{1}{m} \sum_{i=1}^m p_i \|_{TV}. 	
\end{equation} 
Plugging \eqref{eqn:gen_new_task_8} into \eqref{eqn:gen_new_task_7} gives us the desired result.
\subsection{Proof of Corollary \ref{corr:excess_new_task}} \label{proof-corr:excess_new_task}
Note that
\begin{align*}
\E_{\A,\matS}[F_{m+1}(w_\matS)] - \min_\W F_{m+1}   \leq
&\bigg( \E_{\A,\matS}[F_{m+1}(w_\matS) - F(w_\matS)] \bigg)
+\bigg ( \E_{\A,\matS}[F(w_\matS)] - \min_\W F \bigg) \\
&+ \bigg(\min_\W F - \min_\W F_{m+1} \bigg),
\end{align*}
where the second term on the right hand side is bounded by $\epsilon$ by assumption, and the first and last term are both bounded by $D(p_{m+1}, \{p_i\}_{i=1}^m)$ based on Theorem \ref{thm:gen_new_task}.
\subsection{Generalization to a task drawn from a distribution of recurring and unseen tasks} \label{proof-thm:gen_new_taskb}
Here we show how our result for generalization to an unseen task can be extended to the case that the task at test time is generated from a distribution $\pi$ over both recurring tasks $\{\T_i\}_{i=1}^m$ and the unseen task $\T_{m+1}$.
\begin{corollary}\label{corollary:gen_new_task}
Under the premise of Theorem \ref{thm:gen_new_task}, and if the task at the test time is generated from the distribution $\pi$ over $\{\T_i\}_{i=1}^{m+1}$, we have
\begin{align*}
\left | \E_{\pi}[F_i(w)] - F(w) \right | \leq \pi(\T_{m+1}) ~ D(p_{m+1}, \{p_j\}_{j=1}^m) 
 (1-\pi(\T_{m+1})) \sum_{i=1}^m |\pi(\T_{i}) - \frac{1}{m}| ~ D(p_{i}, \{p_j\}_{j=1}^m),
\end{align*}
where $\pi(\T_i)$ is the probability of task $\T_i$ according to distribution $\pi$.
\end{corollary}
\begin{proof}
Note that
\begin{align}\label{eqn:gen_new_C_1}
|\E_\pi[F_i(w)] - F(w)|
\leq
\pi_{m+1} |F_{m+1} - F(w)| + (1-\pi_{m+1}) \sum_{i=1}^m \pi_i |F_i(w) - F(w)|.
\end{align}
Note that by Theorem \ref{thm:gen_new_task} we have
\begin{equation*}
|F_{m+1} - F(w)| \leq D(p_{m+1}, \{p_j\}_{j=1}^m), \quad |F_{i} - F(w)| \leq D(p_{i}, \{p_j\}_{j=1}^m) .   
\end{equation*}
Plugging these into \eqref{eqn:gen_new_C_1} completes the proof.
\end{proof}

\section{Limitations of the algorithmic stability analysis}\label{sec:stability_limit}
Upon reviewers' suggestion, we briefly discuss why the algorithmic stability technique does not lead to meaningful generalization results for nonconvex loss functions. The main issue with applying the stability framework for the nonconvex case is that we have to select a small stepsize to obtain reasonable generalization bounds, but with such small stepsizes, we cannot guarantee that we will find a first-order stationary point (FOSP) solution of the empirical loss in polynomial time. 

To be more precise, consider Theorem 3.12 in Section 3.5 of \cite{hardt2016train}. There, the authors assume the stepsize $\alpha_t$ satisfies the condition $\alpha_t \leq c/t$. To see how this prohibits us from finding an FOSP efficiently, let us recall the convergence analysis of a non-convex smooth objective function $f$. There the main inequality is the following (see Section 1.2.3 in \cite{hardt2016train}):
$$f(w_T) - f^* \geq \sum_{t=0}^T \alpha_t (1-\alpha_t L/2) \| \nabla f(w_t) \|^2,$$
where $L$ is the smoothness parameter and $w_t$ is $t$-th iterate. It can be shown that by setting the stepsize to $\alpha_t = \Theta(1/t)$, as suggested by \cite{hardt2016train}, we would require  $\exp(\Theta(1/\epsilon^2))$ iterations to find an $\epsilon$-FOSP. However, with a constant stepsize, we can achieve the significantly improved rate of $\mathcal{O}(1/\epsilon^2)$ which matches the lower bound for this setting. As this argument shows, to obtain a meaningful generalization bound using algorithmic stability the stepsize should be selected much smaller than the required threshold and as a result the overall iteration/sample complexity could be very large.

Considering this discussion, the algorithmic stability technique imposes a very restrictive assumption on the stepsizes in the \textit{nonconvex setting} which has a detrimental effect on the training error analysis. 
\section{A toy example}
In this section, we provide a simple numerical experiment to validate our theoretical results. We consider a linear regression problem with dimension $d=10$ for the case that we have $m$ tasks and $n$ samples per task. For each task $i$, the feature vector $x$ is drawn according to a normal distribution of $\mathcal{N}(\mu_i, 0.2 I_d)$, where $\mu_i$ is a vector uniformly at random drawn from $[0,1]^d$. In addition, for a given $x$, the label $y$ is given by 
$
y = a_i^\top x + \epsilon_i,    
$
where $\epsilon_i \sim \mathcal{N}(0,0.1)$ and $a_i$ is a random vector. To make tasks similar, we generate the vectors $a_i$ according to $a_i = \frac{u_i + 1_d}{\|u_i + 1_d\|}$, where $u_i$ is a random vector, uniformly drawn from $[0,1]^d$, and $1_d$ is the all-one vector.  

For the loss function, we consider quadratic loss with quadratic regularization, i.e., $l(w, (x,y)) = (w^\top x - y)^2 + \lambda \|w\|^2$, with  $\lambda = 0.01$. We choose the number of samples in the stochastic gradient for adaptation as $K=5$ and the test time learning rate $\alpha = 0.1$, and run MAML for $T=20000$ iterations. 

Figure \ref{fig:recurring} shows the dependence of test error over recurring tasks on $m$ and $n$. In this case the task at test time is a recurring task. We see that the error decreases as $m$ or $n$ increases which is consistent with our theoretical results. 

Next, we consider the case that the task at test time is new and unseen. Note that, in this case, from our theoretical results we know that the error bound includes a term $ D(p_{m+1}, \{p_i\}_{i=1}^m)$ which does not decay with $n$. However, if the distributions are close, this term could be relatively small if $m$ is sufficiently large. To study this matter in our example, we consider two cases:
\begin{itemize}
    \item First, we assume this new task is similar to the observed tasks in training. More formally, similar to the first $m$ tasks, we take $a_{m+1} = \frac{u_{m+1} + 1_d}{\|u_i + 1_d\|}$, where $u_{m+1}$ is again a random vector, uniformly drawn from $[0,1]^d$. Figure \ref{fig:new_similar} shows the test error in this case. As we expected, here we do not gain that much from increasing $n$, but the error decreases as $m$ increases. This matches our intuition, as for small $m$, i.e., $m=1$, the distance between two distributions $p_1$ and $p_2$ could be large. However, as $m$ increases, we have tasks where their distributions are close to $p_{m+1}$, and hence the average distance between distributions $p_i,\cdots, p_m$ and $p_{m+1}$ decreases. 
    \item Second, we make this new task less similar to the observed ones. To do so, this time, we choose $a_{m+1} = \frac{u_{m+1} - 1_d}{\|u_i - 1_d\|}$. In this case, we expect to see a relatively large error which does not decrease with either $m$ or $n$, and Figure \ref{fig:new} exactly shows this matter.
\end{itemize}

\begin{figure}[t!]
\centering
\begin{subfigure}[b]{.45\textwidth}
  \centering
  \includegraphics[width=\linewidth]{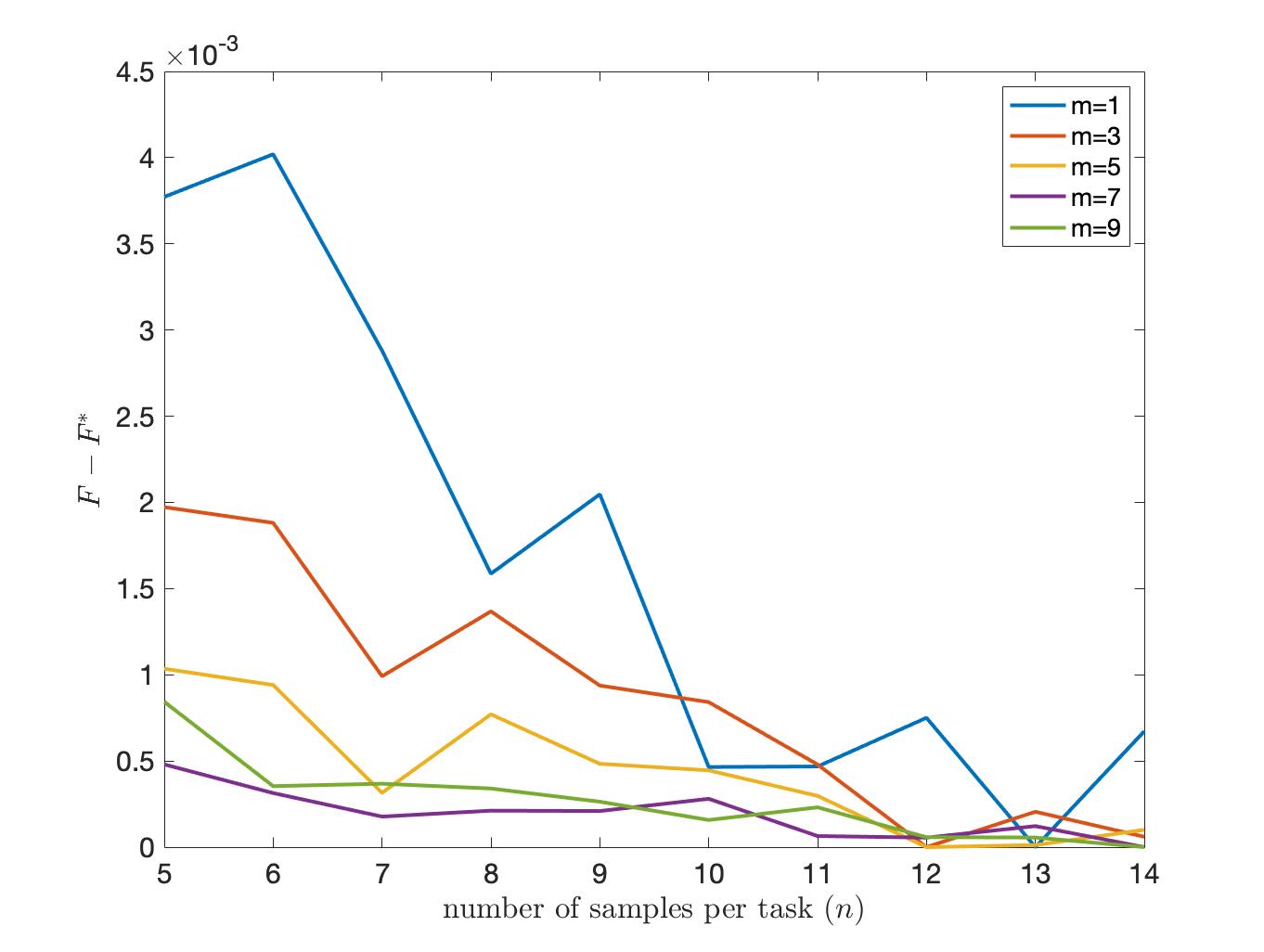}
  \caption{Test error as a function of $n$ for different $m$}
  \label{fig_recurring1}
\end{subfigure}%
\begin{subfigure}[b]{.45\textwidth}
  \centering
  \includegraphics[width=\linewidth]{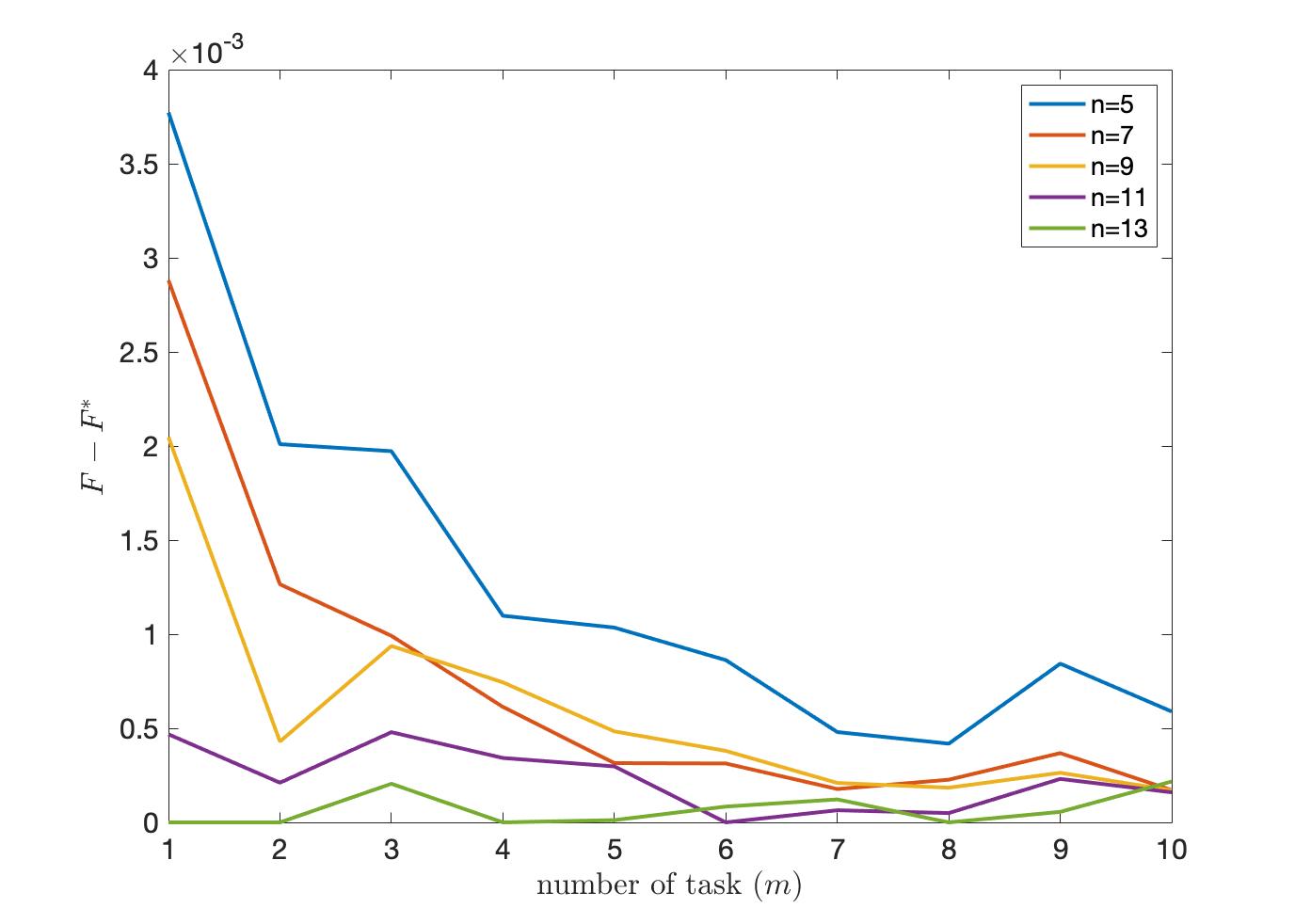}
  \caption{Test error as a function of $m$ for different $n$}
  \label{fig_recurring2}
\end{subfigure}
\vspace{-2mm}
\caption{Test error over recurring tasks}
\label{fig:recurring}
\vspace{-4mm}
\end{figure}

\begin{figure}[t!]
\centering
\begin{subfigure}[b]{.45\textwidth}
  \centering
  \includegraphics[width=\linewidth]{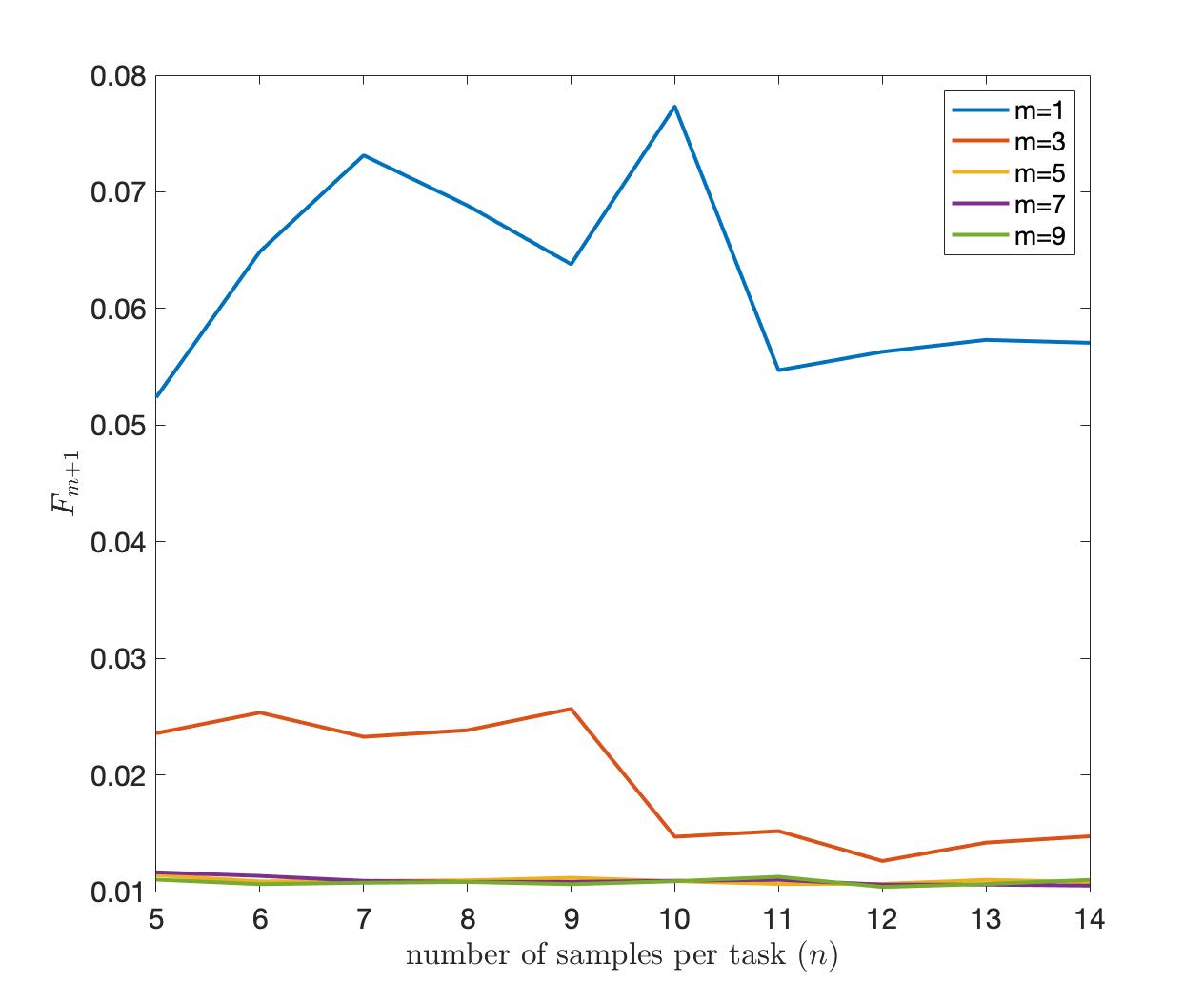}
  \caption{Test error as a function of $n$ for different  $m$}
  \label{fig_new_similar1}
\end{subfigure}%
\begin{subfigure}[b]{.45\textwidth}
  \centering
  \includegraphics[width=\linewidth]{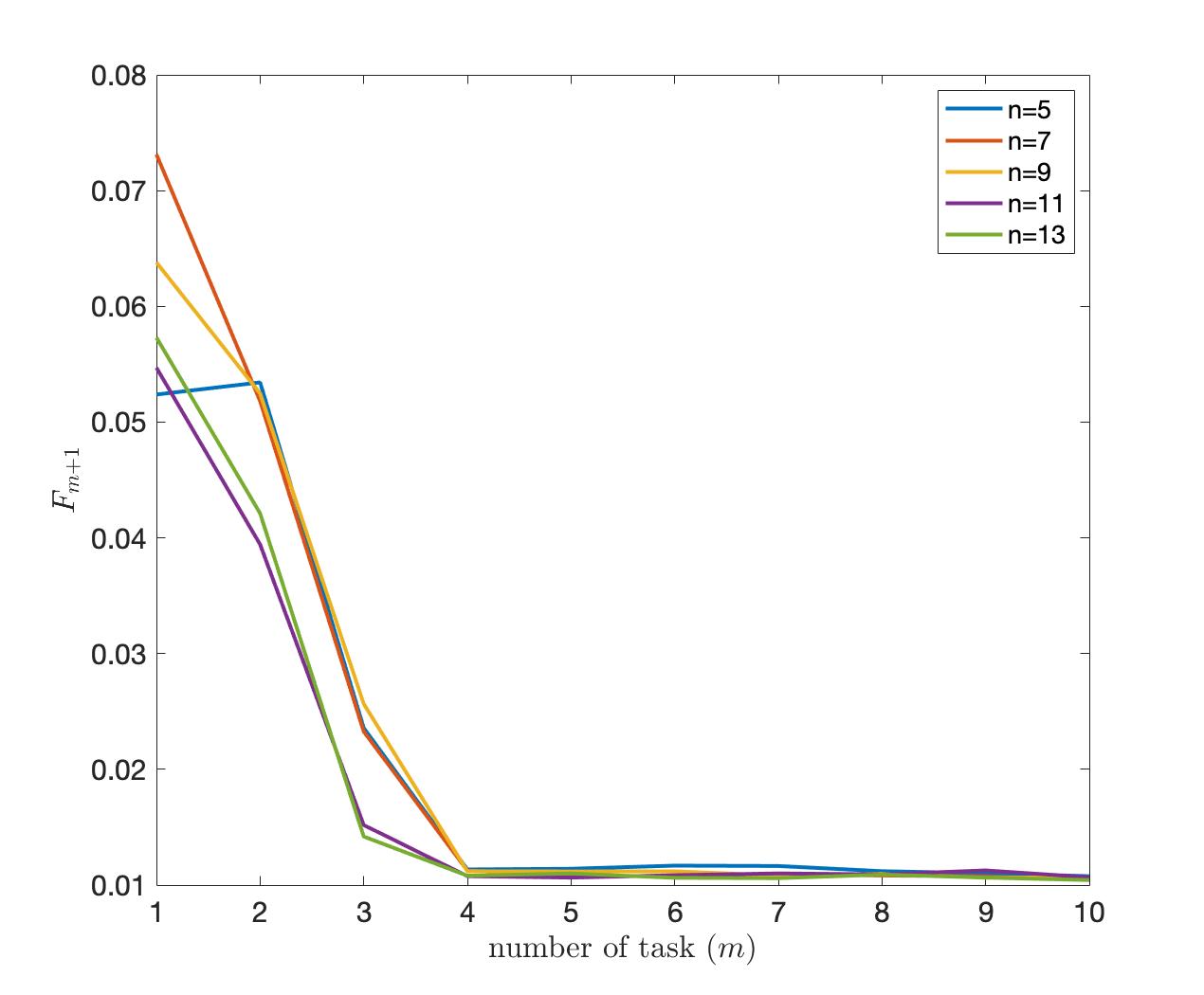}
  \caption{Test error as a function of $m$ for different $n$}
  \label{fig_new_similar2}
\end{subfigure}
\vspace{-2mm}
\caption{Test error over a new but similar task}
\label{fig:new_similar}
\vspace{-4mm}
\end{figure}

\begin{figure}[t!]
\centering
\begin{subfigure}[b]{.45\textwidth}
  \centering
  \includegraphics[width=\linewidth]{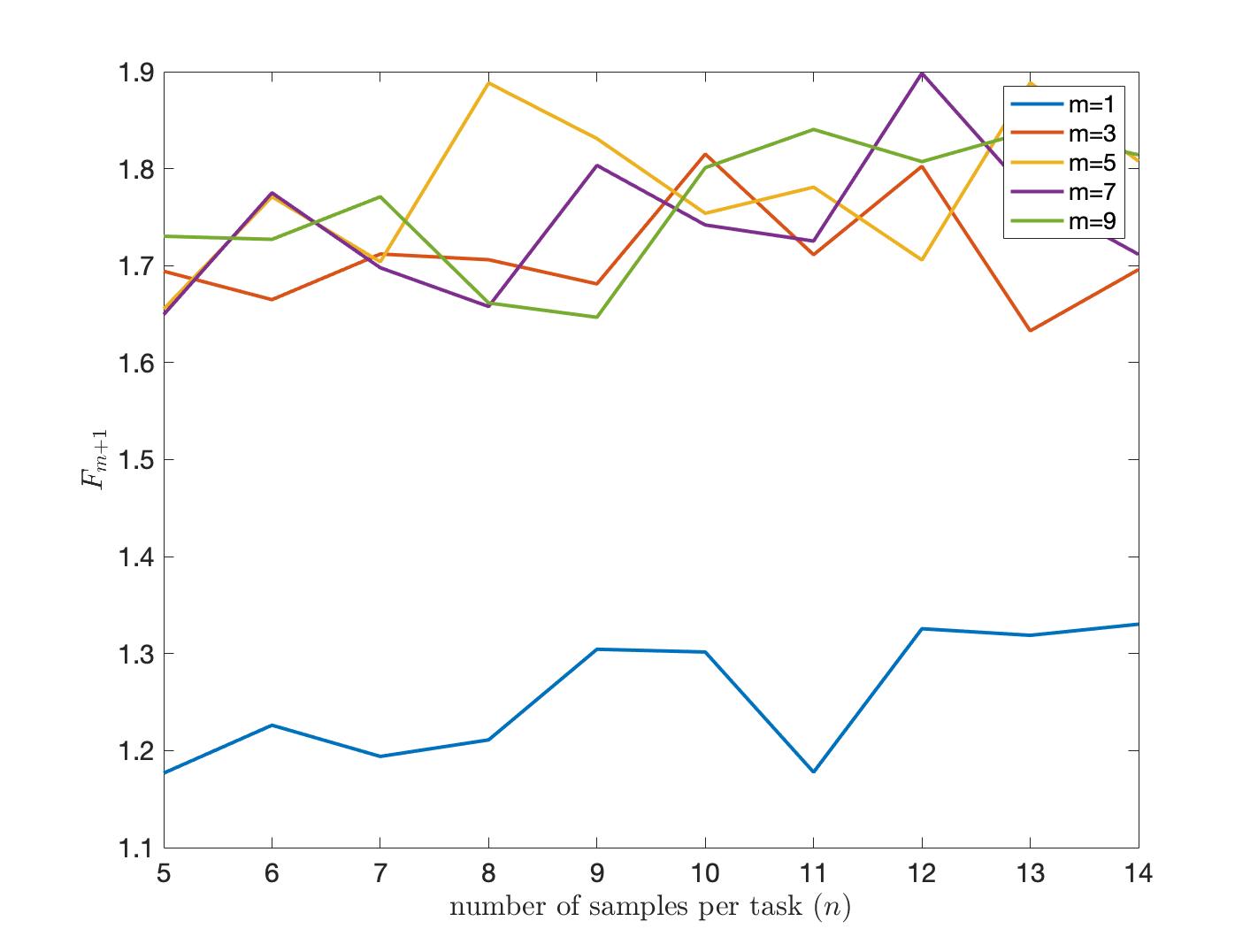}
  \caption{Test error as a function of $n$ for different $m$}
  \label{fig_new1}
\end{subfigure}%
\begin{subfigure}[b]{.45\textwidth}
  \centering
  \includegraphics[width=\linewidth]{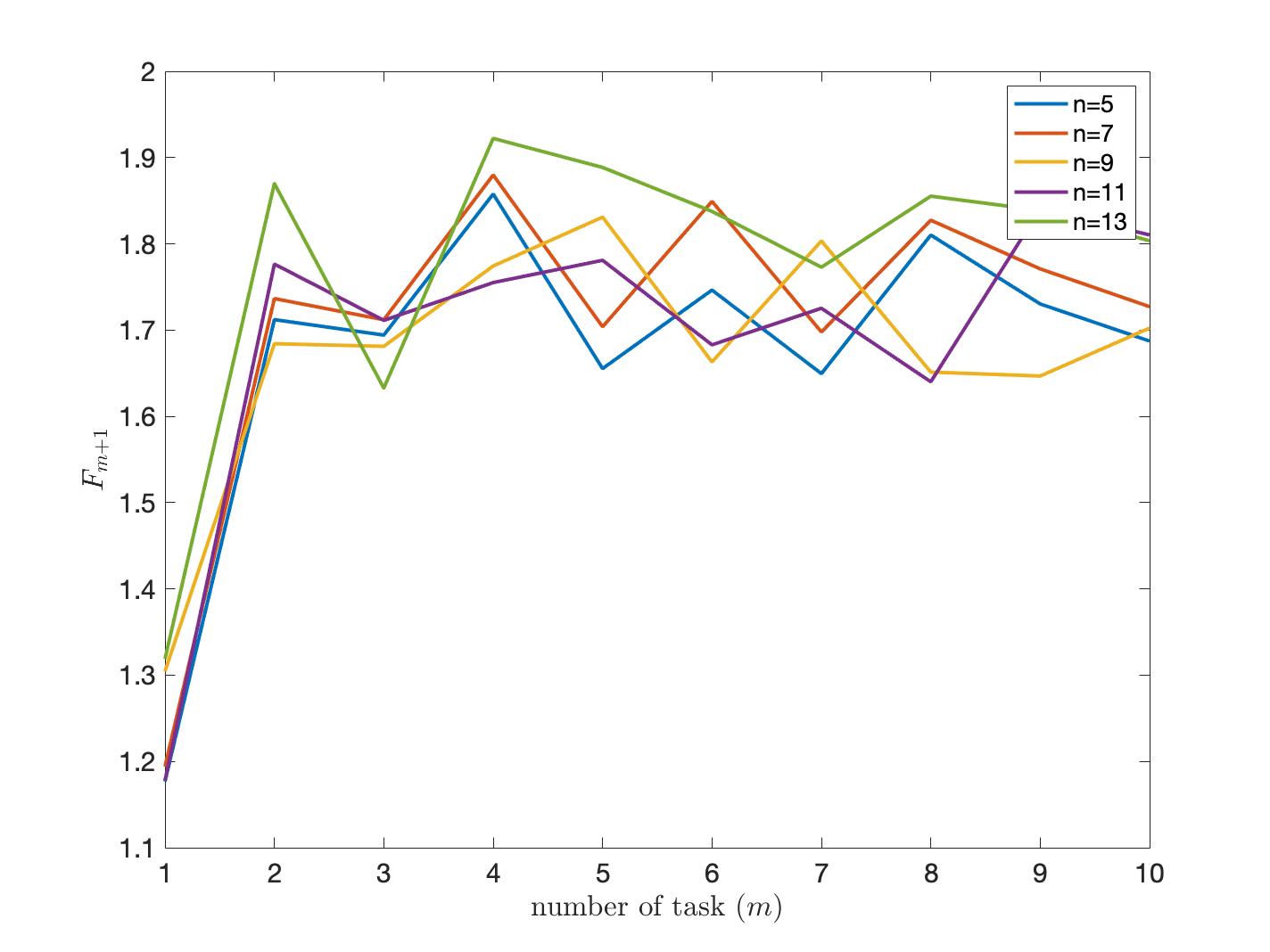}
  \caption{Test error as a function of $m$ for different $n$}
  \label{fig_new2}
\end{subfigure}
\vspace{-2mm}
\caption{Test error over a new and less similar task}
\label{fig:new}
\vspace{-4mm}
\end{figure}

\end{document}